\documentclass[12pt]{article}
\usepackage[margin=1in]{geometry}

\usepackage{makecell,colortbl}
\usepackage{microtype}
\usepackage{graphicx}
\usepackage{subfigure}
\usepackage{booktabs} % for professional tables

\usepackage{times,rotating}
\usepackage{url}
\usepackage{booktabs, multicol, multirow}
\usepackage{morenotations2}
\usepackage{caption}

\usepackage{hyperref}
\usepackage[capitalize,noabbrev]{cleveref}
\title{\papertitle}
\author{
  Richard Nock \qquad Ehsan Amid \qquad Manfred K. Warmuth \\
 Google Research\\
{\normalsize $\{$richardnock,eamid,manfred$\}$@google.com} \\
}
\begin{document}

\date{}

\maketitle

\begin{abstract}
  One of the most popular ML algorithms, \adaboost, can be
derived from the dual of a relative entropy
minimization problem subject to the fact that the positive weights
on the examples sum to one. Essentially, harder examples receive higher probabilities. We generalize this setup to the recently introduced {\it tempered
exponential measure}s (\acrotem s) where normalization is enforced on a specific power of the measure and not the measure itself.
\acrotem s are indexed by a parameter $t$ and generalize exponential families ($t=1$). Our algorithm, \tadaboost, recovers \adaboost~as a special case ($t=1$). We show that \tadaboost~retains \adaboost's celebrated exponential convergence rate when $t\in [0,1)$ while allowing a slight improvement of the rate's hidden constant compared to $t=1$. \tadaboost~partially computes on a generalization of classical arithmetic over the reals and brings notable properties like guaranteed bounded leveraging coefficients for $t\in [0,1)$. From the loss that \tadaboost~minimizes (a generalization of the exponential loss), we show how to derive a new family of \textit{tempered} losses for the induction of domain-partitioning classifiers like decision trees. Crucially, strict properness is ensured for all while their boosting rates span the full known spectrum. Experiments using \tadaboost +trees display that significant leverage can be achieved by tuning $t$.
\end{abstract}

\section{Introduction}\label{sec-intro}

\adaboost~is one of the most popular ML algorithms \cite{fsAD,ssIB}. It efficiently aggregates weak hypotheses
into a highly accurate linear combination \cite{kTO}. 
The common motivations of boosting algorithms focus on choosing good linear weights (the leveraging coefficients) for combining the weak hypotheses. A dual view of boosting highlights the dual parameters, which are the weights on the examples. These weights define a distribution, and \adaboost~can be viewed as minimizing a relative entropy to the last distribution subject to a linear constraint introduced by the current hypothesis \cite{kwBA}. For this reason (more in Section \ref{sec-rel}), \adaboost's weights define an exponential family.

\textbf{In this paper}, we go beyond weighing the examples with a discrete exponential family distribution, relaxing the constraint that the total mass be unit but instead requiring it for the measure's $1/(2-t)$'th power, where $t$ is a temperature parameter. Such measures, called {\it tempered exponential measures} (\acrotem s), have been recently introduced \cite{anwCA}. Here we apply the discrete version of these \acrotem s for deriving a novel boosting algorithm called \tadaboost. Again the measures are solutions to a relative entropy minimization problem, but the relative entropy is built from Tsallis entropy and  ``tempered'' by a parameter $t$. As $t\rightarrow 1$ \acrotem s become standard exponential family distributions and our new algorithm merges into \adaboost. As much as \adaboost~minimizes the exponential loss, \tadaboost~minimizes a generalization of this loss we denote as the \textit{tempered exponential loss}.

\acrotem s were introduced in the context of clustering, where they were shown to improve the robustness to outliers of clustering's population minimizers \cite{anwCA}. Boosting is a high-precision machinery: \adaboost~is known to achieve near-optimal boosting rates under the weak learning assumption \cite{aghmBS}, but it has long been known that numerical issues can derail it, in particular, because of the unbounded weight update rule \cite{kIA}. So the question of what the \acrotem~setting can bring for boosting is of primordial importance. As we show, \tadaboost~can suffer no rate setback as boosting's exponential rate of convergence can be preserved for all $t \in [0,1)$. Interestingly, however, for such a range of $t$, the weight update becomes bounded. \tadaboost~makes use of a generalization of classical arithmetic over the reals introduced decades ago \cite{nlwGA} and besides linear separators, it can also learn progressively clamped models\footnote{Traditionally, clamping a sum is done after it has been fully computed. In our case, it is clamped after each new summand is added.}. Also, the weight update makes appear a new regime whereby weights can "switch off and on": an example's weight can become zero if too well classified by the current linear separator, and later on revert to non-zero if badly classified by a next iterate.

Boosting algorithms for linear models like \adaboost~bring more than just learning good linear separators: it is known that (ada)boosting linear models can be used to emulate the training of \textit{decision trees} (DT) \cite{mnwRC}, which are models known to lead to some of the best of-the-shelf classifiers when linearly combined \cite{fhtAL}. Unsurprisingly, the algorithm obtained emulates the classical top-down induction of a tree found in major packages like CART \cite{bfosCA} and C4.5 \cite{qC4}. The \textit{loss} equivalently minimized, which is, \textit{e.g.}, Matusita's loss for \adaboost~\cite[Section 4.1]{ssIB}, is a lot more consequential. Contrary to losses for real-valued classification, losses to train DTs rely on the estimates of the posterior learned by the model; they are usually called \textit{losses for Class Probability Estimation} (CPE \cite{rwCB}). The CPE loss is crucial to elicit because (i) it is important to check whether it is proper (Bayes rule is optimal for the loss \cite{sEO}), and (ii) it conditions boosting rates, only a handful of them being known, for the most popular CPE losses \cite{kmOT,nwLO,snsBP}.

\textbf{In this paper}, we show that this emulation scheme on \tadaboost~provides a new family of CPE losses with remarkable constancy with respect to properness: losses are \textit{strictly} proper (Bayes rule is the \textit{sole} optimum) for any $t\in (-\infty, 2)$ and proper for $t=-\infty$. Furthermore, over the range $t\in [-\infty, 1]$, the range of boosting rates spans the full spectrum of known boosting rates \cite{kmOT}. 

We provide experiments displaying the boosting ability of \tadaboost~over a range of $t$ encompassing potentially more than the set of values covered by our theory, and highlight the potential of using $t$ as a parameter for efficient tuning the loss \cite[Section 8]{rwCB}. For the sake of readability, proofs are relegated to the appendix (\supplement). A primer on \acrotem s is also given in \supplement, Section \ref{sec-sup-primer}.

\section{Related work}\label{sec-rel}

Boosting refers to the ability of an algorithm to combine the outputs of moderately accurate, "weak" hypotheses into a highly accurate, "strong" ensemble. Originally, boosting was introduced in the context of Valiant's PAC learning model as a way to circumvent the then-existing amount of related negative results \cite{kTO,vAT}. After the first formal proof that boosting is indeed achievable \cite{sTS}, \adaboost~became the first practical and proof-checked boosting algorithm \cite{fsAD,ssIB}. Boosting was thus born in a machine learning context, but later on, it also emerged in statistics as a way to learn from class residuals computed using the gradient of the loss \cite{fhtAL,nnTP}, resulting this time in a flurry of computationally efficient algorithms, still called boosting algorithms, but for which the connection with the original weak / strong learning framework is in general not known.

Our paper draws its boosting connections with \adaboost's formal lineage. \adaboost~has spurred a long line of work alongside different directions, including statistical consistency \cite{btAI}, noise handling \cite{lsRC,mnwRC}, low-resource optimization \cite{nwLO}, \emph{etc}. The starting point of our work is a fascinating result in convex optimization establishing a duality between the algorithm and its memory of past iteration's performances given by the probability distribution 
%of so-called \textit{weights} 
over the examples \cite{kwBA}. From this standpoint, \adaboost~solves the dual of the optimization of a relative entropy between the new and current distribution subject to a linear constraint on the weak classifier's performance. Whenever a relative entropy is minimized subject to linear constraints, then the solution is a member of an exponential family of distributions (see \textit{e.g.} \cite[Section 2.8.1]{aIG} for an axiomatization of exponential families). Indeed \adaboost's distribution on the examples is a member of a discrete exponential family where the training examples are the finite support of the distribution, sufficient statistics are defined from the weak learners, and the leveraging coefficients are the natural parameters. In summary, there is an intimate relationship between boosting \`a-la-\adaboost, exponential families, and Bregman divergences \cite{kwBA,cssLR,nnOT} and our work
"elevates" these methods above exponential families.

\section{Definitions}\label{sec-defs-setting}

We define the $t$-logarithm and $t$-exponential,
\begin{eqnarray}
\log_t (z) \defeq \frac{1}{1-t}\cdot\left(z^{1-t}-1\right) & , & \exp_t (z) \defeq \left[1+(1-t) z\right]^{1/(1-t)}_+ \quad ([z]_+ \defeq \max\{0,z\})\label{defExpLogT},
\end{eqnarray}
where the case $t=1$ is supposed to be the extension by continuity to the $\log$ and $\exp$ functions, respectively. To preserve the concavity of $\log_t$ and the convexity of $\exp_t$, we need $t\geq 0$. In the general case, we also note the asymmetry of the composition: while $\exp_t \log_t(z) = z, \forall t \in \mathbb{R}$, we have $\log_t \exp_t(z) = z$ for $t=1$ ($\forall z\in \mathbb{R}$), but
\begin{eqnarray*}
\log_t \exp_t(z) = \max\left\{-\frac{1}{1-t}, z\right\} \quad (t<1) & \mathrm{and} & \log_t \exp_t(z) = \min\left\{\frac{1}{t-1}, z\right\} \quad (t>1).
\end{eqnarray*}
Comparisons between vectors and real-valued functions written on vectors are assumed component-wise. We assume $t\neq 2$ and define notation $t^* \defeq 1/(2-t)$. We now define the key set in which we model our weights (boldfaces denote  vector notation).
\begin{definition}\label{defCOSIMPLEX}
The co-simplex of $\mathbb{R}^m$, $\tilde{\Delta}_m$ is defined as $\tilde{\Delta}_m \defeq \{\ve{q}\in \mathbb{R}^m: \ve{q} \geq \ve{0} \wedge \ve{1}^\top \ve{q}^{1/t^*} = 1\}$.
\end{definition}
%To work out the technical content of the paper, it is sufficient to assume that $\tilde{\Delta}_m$ indeed defines a set of tempered exponential measures (\acrotem s), pretty much like it would instead define a regular distribution if we removed the exponent ${1/t^*}$. The \supplement, Section \ref{sec-sup-primer}, defines \acrotem s and makes explicit the connection with $\tilde{\Delta}_m$. 
The letters $\ve{q}$ will be used to denote \acrotem s 
in $\tilde{\Delta}_m$ while $\ve{p}$ denote the co-density $\ve{q}^\frac{1}{t^*}$ or any element of the probability simplex. We define the general tempered relative entropy as
\begin{eqnarray}
  D_t(\ve{q}' \| \ve{q}) & \defeq & \sum_{i\in [m]} {q}_i' \cdot \left(\log_t {q}_i'  - \log_t {q}_i \right) - \log_{t-1} {q}_i' + \log_{t-1} {q}_i,
\end{eqnarray}
where $[m] \defeq \{1,..., m\}$. The tempered relative entropy is a Bregman divergence with convex generator $\phi_t(z) \defeq z \log_t z - \log_{t-1}(z)$ (for $t\in \mathbb{R}$) and $\phi_t(z)'=\log_t(x)$. As $t\rightarrow 1$, $D_t(\ve{q},\ve{q}')$ becomes the relative entropy with generator $\phi_1(x)=x \log(x) - x$.

\section{Tempered boosting as tempered entropy projection}\label{sec-defs-tboost}

We start with a fixed sample $\mathcal{S}= \{(\ve{x}_i,y_i):i\in[m]\}$ where observations $\ve{x}_i$ lie in some domain $\mathcal{X}$ and labels $y_i$ are $\pm 1$. \adaboost~maintains a distribution $\ve{p}$ over the sample. At the current iteration, this distribution is updated based on a current \textit{weak hypothesis} $h\in \mathbb{R}^{\mathcal{X}}$ using an exponential update:
$$p_i' = 
\frac{p_{i} \cdot \exp(-\mu u_{i})}{\sum_k p_{k} \cdot \exp(-\mu u_{k})}, \text{ where }u_{i}\defeq y_i h(\ve{x}_i).$$
In \cite{kwBA} this update is motivated as minimizing a relative entropy subject to the constraint that $\ve{p}'$ is a distribution summing to 1 and $\ve{p}'^\top \ve{u}=0$.
Following this blueprint, we create a boosting algorithm maintaining a discrete \acrotem~over the sample which is motivated as a constrained minimization of the tempered relative entropy, with a normalization constraint on the co-simplex of $\mathbb{R}^m$:
\begin{eqnarray}
\ve{q}' & \defeq & \arg\min_{\begin{array}{c}\widetilde{\ve{q}} \in \tilde{\Delta}_m\\\widetilde{\ve{q}}^\top \ve{u} = 0\end{array}} D_t(\widetilde{\ve{q}} \| \ve{q}), \quad \mbox{ with } \ve{u} \in \mathbb{R}^m. \label{defBATEP}
\end{eqnarray}
We now show that the solution $\ve{q}'$ is a tempered generalization of \adaboost's exponential update.
\begin{theorem}\label{thENTBOOST}
For all $t \in \mathbb{R}\backslash \{2\}$, all solutions to \eqref{defBATEP} have the form
  \begin{eqnarray}
q_{i}' = \frac{\exp_t(\log_t q_{i} - \mu  u_{i})}{{Z}_{t}}  \quad \left( = \frac{q_{i} \otimes_t \exp_t(- \mu  u_{i})}{{Z}_{t}},  \mbox{ with } a\otimes_t b \defeq [a^{1-t} + b^{1-t} - 1]^{\frac{1}{1-t}}_+\right), \label{defWoptM}
  \end{eqnarray}
  where ${Z}_{t}$ ensures co-simplex normalization of the co-density. Furthermore, the unknown $\mu$ satisfies
  \begin{eqnarray}
    \mu \in \arg\max -\log_t ({Z}_{t}(\mu)) \quad (= \arg\min {Z}_{t}(\mu)),\label{defLossFromZ}
  \end{eqnarray}
 or equivalently is a solution to the nonlinear equation
\begin{eqnarray}
    \ve{q}'(\mu)^\top \ve{u} & = & 0.\label{eqMU}
\end{eqnarray}
Finally, if either (i) $t \in \mathbb{R}_{>0}\backslash \{2\}$ or (ii) $t=0$ and $\ve{q}$ is not collinear to $\ve{u}$, then ${Z}_{t}(\mu)$ is strictly convex: the solution to \eqref{defBATEP} is thus unique, and can be found from expression \eqref{defWoptM} by finding the unique minimizer of \eqref{defLossFromZ} or (equivalently) the unique solution to \eqref{eqMU}.
  \end{theorem}
  (Proof in \supplement, Section \ref{proof-thENTBOOST}) The $t$-product $\otimes_t$, which satisfies $\exp_t(a+b) = \exp_t(a) \otimes_t \exp_t(b)$, was introduced in \cite{nlwGA}. Collinearity never happens in our ML setting because $\ve{u}$ contains the edges of a weak classifier: $\ve{q} > 0$ and collinearity would imply that $\pm$ the weak classifier performs perfect classification, and thus defeats the purpose of training an ensemble. $\forall t \in \mathbb{R}\backslash \{2\}$, we have the simplified expression for the normalization coefficient of the \acrotem~and the co-density $\ve{p}'$ of $\ve{q}'$:
  \begin{eqnarray}
 {Z}_{t} = \left\| \exp_t\left(\log_t \ve{q} - \mu \cdot \ve{u}\right) \right\|_{1/t^*} &\!\!\! ;\!\! & p_i' = \frac{p_{i} \otimes_{t^*} \exp_{t^*}\left(- \frac{\mu  u_{i}}{t^*}\right)}{Z'_{t}} \quad \!\!\!\left(\mbox{ with }Z'_{t} \defeq {Z}^{1/t^*}_{t} \right).\label{defWCODoptM}
\end{eqnarray}

\section{Tempered boosting for linear classifiers and clamped linear classifiers}\label{sec-defs-tada}

\begin{algorithm}[t]
\caption{\tada $(t,\mathcal{S},J)$}\label{tadalgo}
\begin{algorithmic}
  \STATE  \textbf{Input:} $t\in [0,1]$, training sample $\mathcal{S}$, $\#$iterations $J$;
  \STATE  \textbf{Output:} classifiers $H_J,H^{(\nicefrac{1}{1-t})}_J$ (see \eqref{defLinClin});
  \STATE  Step 1 : initialize tempered weights: $\ve{q}_1 = (1/m^{t^*}) \cdot \ve{1} \quad (\in \tilde{\Delta}_m)$;
  \STATE  Step 2 : for $j=1, 2, ..., J$
  \STATE \hspace{1cm} Step 2.1 : get weak classifier $h_j \leftarrow \mbox{weak$\_$learner}(\ve{q}_j, \mathcal{S})$;
  \STATE \hspace{1cm} Step 2.2 : choose weight update coefficient $\mu_j \in \mathbb{R}$;
  \STATE \hspace{1cm} Step 2.3 : $\forall i \in [m]$, for $u_{ji} \defeq y_i h_j(\ve{x}_i)$, update the tempered weights as
\begin{eqnarray}
q_{(j+1)i} = \frac{q_{ji} \otimes_t \exp_t(- \mu_j  u_{ji})}{{Z}_{tj}}, \quad \text{where }{Z}_{tj} = \left\| \ve{q}_j \otimes_t \exp_t(- \mu_j  \ve{u}_{j}) \right\|_{1/t^*}\label{temperedweights}.
    \end{eqnarray}
  \STATE \hspace{1cm} Step 2.4 : choose leveraging coefficient $\alpha_j \in \mathbb{R}$;
\end{algorithmic}
\end{algorithm}

\paragraph{Models} A model (or classifier) is an element of $\mathbb{R}^{\mathcal{X}}$. For any model $H$, its empirical risk over $\mathcal{S}$ is $\zoloss(H,\mathcal{S}) \defeq (1/m) \cdot \sum_{i}\iver{y_i \neq \mathrm{sign}(H(\ve{x}_i))}$ where $\iver{.}$, Iverson's bracket \cite{kTN}, is the Boolean value of the inner predicate. We learn linear separators and \textit{clamped} linear separators. Let $(v_j)_{j\geq 1}$ be the terms of a series and $\delta \geq 0$. The clamped sum of the series is:
\begin{eqnarray*}
\sideset{_{(-\delta)}^{(\delta)}}{}\sum_{j\in [J]} v_j & \defeq & \min\left\{\delta,\max\left\{-\delta,v_J + \sideset{_{(-\delta)}^{(\delta)}}{}\sum_{j\in [J-1]} v_j \right\}\right\} \quad (\in [-\delta, \delta]), \mbox{ for $J>1$},
\end{eqnarray*}
and we define the base case $J=1$ by replacing the inner clamped sum by 0. Note that clamped summation is non-commutative, and so is different from clamping in $[-\delta, \delta]$ the whole sum itself\footnote{Fix for example $a=-1, b=3, \delta = 2$. For $v_1 = a, v_2 = b$, the clamped sum is $2 = -1 + 3$, but for $v_1 = b, v_2 = a$, the clamped sum becomes $1 = \textbf{2} - 1$.}. Given a set of so-called weak hypotheses $h_j \in \mathbb{R}^{\mathcal{X}}$ and leveraging coefficients $\alpha_j \in \mathbb{R}$ (for $j\in [J]$), the corresponding linear separators and clamped linear separators are
\begin{eqnarray}
H_J(\ve{x}) \defeq \sum_{j\in [J]} \alpha_j h_j(\ve{x}) & ; & H^{(\delta)}_J(\ve{x}) \defeq \sideset{_{(-\delta)}^{(\delta)}}{}\sum_{j\in [J]} \alpha_j h_j(\ve{x}).\label{defLinClin}
\end{eqnarray}

\paragraph{Tempered boosting and its general convergence} Our algorithm, \tada, is presented in Algorithm \ref{tadalgo}. Before analyzing its convergence, several properties are to be noted for \tadaboost: first, it keeps the appealing property, introduced by \adaboost, that examples receiving the wrong class by the current weak classifier are reweighted higher (if $\mu_j > 0$). Second, the leveraging coefficients for weak classifiers in the final classifier ($\alpha_j$s) are not the same as the ones used to update the weights ($\mu_j$s), unless $t=1$. Third and last, because of the definition of $\exp_t$ \eqref{defExpLogT}, if $t<1$, tempered weights can switch off and on, \textit{i.e.}, become 0 if an example is "too well classified" and then revert back to being $> 0$ if the example becomes wrongly classified by the current weak classifier (if $\mu_j > 0$). To take into account those zeroing weights, we denote $[m]_j^\dagger \defeq \{i : q_{ji} = 0\}$ and $m_j^\dagger \defeq \mathrm{Card}([m]_j^\dagger)$ ($\forall j \in [J]$). Let $R_j \defeq \max_{i \not\in [m]_j^\dagger} |y_i h_j(\ve{x}_i)| / q_{ji}^{1-t}$ and ${q^\dagger_j} \defeq \max_{i \in [m]_j^\dagger} |y_i h_j(\ve{x}_i)|^{1/(1-t)} / R^{1/(1-t)}_j$. It is worth noting that ${q^\dagger_j}$ is homogeneous to a tempered weight. 
  \begin{theorem}\label{thBOOST1}
    At iteration $j$, define the weight function ${q}'_{ji} \defeq q_{ji}$ if $i \not\in [m]_j^\dagger$ and ${q^\dagger_j}$ otherwise; set
    \begin{eqnarray}
 \rho_{j} & \defeq & \frac{1}{(1+m_j^\dagger {q^\dagger_j}^{2-t})R_j} \cdot \sum_{i\in [m]} q'_{ji} y_i h_j(\ve{x}_i) \quad (\in [-1,1]).\label{defRHO}
    \end{eqnarray}
    In algorithm \tada, consider the choices (with the convention $\prod_{k=1}^{0} v_k \defeq 1$)
    \begin{eqnarray}
 \mu_j \defeq -\frac{1}{R_j} \cdot \log_t \left(\frac{1-\rho_{j}}{M_{1-t}(1-\rho_{j},1+\rho_{j})}\right) & , & \alpha_j \defeq m^{1-t^*}\cdot \left(\prod_{k=1}^{j-1} {Z}_k\right)^{1-t} \cdot \mu_j,\label{defRHOMU}
    \end{eqnarray}
    where $M_q(a,b) \defeq ((a^q+b^q)/2)^{1/q}$ is the $q$-power mean. Then for any $H \in \{H_J,H^{(\nicefrac{1}{1-t})}_J\}$, its empirical risk is upperbounded as:
    \begin{eqnarray}
      \zoloss(H,\mathcal{S}) \leq \prod_{j=1}^J {Z}^{2-t}_{tj} \leq \prod_{j=1}^J \left(1+m_j^\dagger {q^\dagger_j}^{2-t}\right) \cdot K_t(\rho_j) \quad \left(K_t(z) \defeq \frac{1-z^2}{M_{1-t}(1-z,1+z)}\right).\label{eqGua01}
      \end{eqnarray}
    \end{theorem}
    (Proof in \supplement, Section \ref{proof-thBOOST1}) We jointly comment \tadaboost~and Theorem \ref{thBOOST1} in two parts.
    \paragraph{Case $t\rightarrow 1^-$:} \tadaboost~converges to \adaboost~and Theorem \ref{thBOOST1} to its convergence analysis: \tada~converges to \adaboost~as presented in \cite[Figure 1]{ssIBj}: the tempered simplex becomes the probability simplex, $\otimes_t$ converges to regular multiplication, weight update \eqref{temperedweights} becomes \adaboost's, $\alpha_j \rightarrow \mu_j$ in \eqref{defRHOMU} and finally the expression of $\mu_j$ converges to \adaboost's leveraging coefficient in \cite{ssIBj} ($\lim_{t\rightarrow 1} M_{1-t}(a,b) = \sqrt{ab}$). Even guarantee \eqref{eqGua01} converges to \adaboost's popular guarantee of \cite[Corollary 1]{ssIBj} ($\lim_{t\rightarrow 1} K_t(z) = \sqrt{1-z^2}$, $m_j^\dagger = 0$). Also, in this case, we learn only the unclamped classifier since $\lim_{t\rightarrow 1^-} H^{(\nicefrac{1}{1-t})}_J = H_J$.
    \setlength{\intextsep}{0pt}
\setlength{\columnsep}{10pt}
\paragraph{Case $t < 1$:} Let us first comment on the convergence rate. The proof of Theorem \ref{thBOOST1} shows that $K_t(z) \leq \exp(-z^2/(2t^*))$. Suppose there is no weight switching, so $m_j^\dagger = 0, \forall j$ (see Section \ref{sec-exp}) and, as in the boosting model, suppose there exists $\gamma > 0$ such that $|\rho_{j}| \geq \gamma, \forall j$. Then \tadaboost~is guaranteed to attain empirical risk below some $\varepsilon > 0$ after a number of iterations equal to $J = (2t^* / \gamma^2) \cdot \log (1/\varepsilon)$.
\begin{figure}[t]
  \begin{center}
    \includegraphics[trim=0bp 0bp 0bp 0bp,clip,width=0.5\columnwidth]{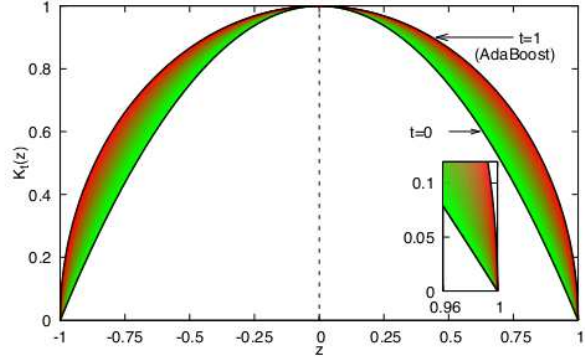}
    \end{center}
  \caption{Plot of $K_t(z)$ in \eqref{eqGua01}, $t\in[0,1]$ (the smaller, the better for convergence).}\label{figRHO}
\end{figure}
$t^*$ being an increasing function of $t \in [0,1]$, we see that \tadaboost~is able to slightly improve upon \adaboost's celebrated rate \cite{tAP}. However, $t^* = 1/2$ for $t=0$ so the improvement is just on the hidden constant. This analysis is suited for small values of $|\rho_{j}|$ and does not reveal an interesting phenomenon for better weak hypotheses. Figure \ref{figRHO} compares $K_t(z)$ curves ($K_1(z) \defeq \lim_{t\rightarrow 1} K_t(z) = \sqrt{1-z^2}$ for \adaboost, see \cite[Corollary 1]{ssIB}), showing the case $t<1$ can be substantially better, especially when weak hypotheses are not "too weak". If $m_j^\dagger > 0$, switching weights can impede our convergence analysis, though we still show that exponential convergence is always possible if $m_j^\dagger {q^\dagger_j}^{2-t}$ is small enough. A good criterion to train weak hypotheses is then the optimization of the edge $\rho_j$, thus using $\ve{q}'_j$ normalized in the simplex.
Other key features of \tadaboost~are as follows. First, the weight update and leveraging coefficients of weak classifiers are bounded because $|\mu_j| < 1/(R_j(1-t))$ (\supplement, Lemma \ref{lemBMUj}) (this is not the case for $t\rightarrow 1^-$). This guarantees that new weights are bounded before normalization (unlike for $t \rightarrow 1^-$). Second, we remark that $\mu_j \neq \alpha_j$ if $t\neq 1$. Factor $m^{1-t^*}$ is added for convergence analysis purposes; we can discard it to train the unclamped classifier: it does not change its empirical risk. This is, however, different for factor $\prod_{k=1}^{j-1} {Z}_k$: from \eqref{eqGua01}, we conclude that this is an indication of how well the past ensemble performs. As it gets better and better, it progressively dampens the leverage of the next weak classifiers, a phenomenon that does not occur in boosting, where an excellent weak hypothesis on the current weights can have a leveraging coefficient so large that it wipes out the classification of the past ones. This can be useful to control numerical instabilities. We also conjecture that a finer analysis would prove a similar phenomenon on margin optimization as, \textit{e.g.}, in \cite{sfblBT}.
    \paragraph{The tempered exponential loss} In the same way as \adaboost~introduced the now famous exponential loss, \eqref{eqGua01} recommends to minimize the normalization coefficient, following \eqref{defWCODoptM},
    \begin{eqnarray}
      {Z}^{2-t}_{tj}(\mu) & = & \left\| \exp_t\left(\log_t \ve{q}_j - \mu \cdot \ve{u}_j\right) \right\|_{1/t^*}^{1/t^*}\quad \left(\mbox{with } u_{ji} \defeq y_i h_j(\ve{x}_i)\right)\label{defTEXPloss}.
    \end{eqnarray}
    In an equivalent form, one can easily show that
    \begin{eqnarray}
      {Z}_{tj}^{2-t}(\mu) & = & \sum_i q^{2-t}_{ji} \cdot \exp^{2-t}_t\left(-\frac{\mu u_{ji}}{q^{1-t}_{ji} }\right)\nonumber\\
      & = & \sum_i p_{ji} \exp^{2-t}_t\left(-\frac{\mu u_{ji}}{p^{1-t^*}_{ji} }\right)\label{defTEXPLOSS},
    \end{eqnarray}
    where we remind that $\ve{p}_j$ is the co-density of $\ve{q}_j$ (Section \ref{sec-defs-setting}), which is in the probability simplex. Hence, \eqref{defTEXPLOSS} is a classical expectation, though the $\exp_t$ also integrates the co-density weights. In boosting, examples that have received the right classification receive smaller weights. What \eqref{defTEXPLOSS} shows is that in this expectation, the contribution of examples with the right class ($\mu u_{ij} > 0$) dimishes even further. We cannot easily unravel the normalization coefficient to make appear a complete classifier, because of the presence of $[.]_+$ in $\exp_t$ \eqref{defExpLogT}. However, if $\max_i |h_j(\ve{x}_i)|$ is small enough for any $j \in [J]$, we easily obtain from \eqref{defTEXPloss} after dropping $j$ and the argument $\mu$ for readability, that we end up minimizing a loss being
    \begin{eqnarray}
\temperedexploss{t}(H, \mathcal{S}) \left(\defeq {Z}^{2-t}_{tj}(\mu)\right) & = & \frac{1}{m} \cdot \sum_i \exp^{2-t}_t\left(- y_i H(\ve{x}_i)\right), \label{defEXPRISK}
    \end{eqnarray}
    where we have absorbed in $H$ the factor $m^{1-t^*}$ appearing in the $\exp_t$ (scaling $H$ by a positive value does not change its empirical risk). One retrieves \adaboost's exponential loss by setting $t\rightarrow 1^-$ in \eqref{defEXPRISK}, so we denote with a slight abuse of language \eqref{defEXPRISK}, and \eqref{defTEXPloss} by extension, the \textit{tempered exponential loss}. Notice that one can choose to minimize $\temperedexploss{t}(H, \mathcal{S})$ disregarding any constraint on $|H|$.

\section{A broad family of boosting-compliant proper losses for decision trees}\label{sec-newc}

\paragraph{Losses for class probability estimation} When it comes to tabular data, it has long been known that some of the best models to linearly combine with boosting are decision trees (DT, \cite{fhtAL}). Decision trees, like other domain-partitioning classifiers, are not trained by minimizing a \textit{surrogate loss} defined over real-valued predictions, but defined over \textit{class probability estimation} (CPE, \cite{rwID}), those estimators being posterior estimation computed at the leaves. Let us introduce a few definitions for those. A CPE loss $\loss : \{-1,1\} \times [0,1]
\rightarrow \mathbb{R}$ is
\begin{eqnarray}
\loss(y,u) & \defeq & \iver{y=1}\cdot \partialloss{1}(u) +
                     \iver{y=-1}\cdot \partialloss{-1}(u). \label{eqpartialloss}
\end{eqnarray}
Functions $\partialloss{1}, \partialloss{-1}$ are called \textit{partial} losses. The pointwise conditional risk of local guess $u \in [0,1]$ with respect to a ground truth $v \in [0,1]$ is: 
\begin{eqnarray}
  \poirisk(u,v) & \defeq & v\cdot \partialloss{1}(u) + (1-v)\cdot \partialloss{-1}(u) \label{eqpoirisk}.
\end{eqnarray}
A loss is \textit{proper} iff for any ground truth $v \in [0,1]$, $\poirisk(v,v) = \inf_u \poirisk(u,v)$, and strictly proper iff $u=v$ is the sole minimizer \cite{rwID}. The (pointwise) \textit{Bayes} risk is $\poibayesrisk(v) \defeq \inf_u \poirisk(u,v)$. The log/cross-entropy-loss, square-loss, Matusita loss are examples of CPE losses. One then trains a DT minimizing the expectation of this loss over leaves' posteriors, $\expect_\leaf [\poibayesrisk(p_\leaf)]$, $p_\leaf$ being the local proportion of positive examples at leaf $\leaf$ -- or equivalently, the local posterior.

\paragraph{Deriving CPE losses from (ada)boosting} Recently, it was shown how to derive in a general way a CPE loss to train a DT from the minimization of a surrogate loss with a boosting algorithm \cite{mnwRC}. In our case, the surrogate would be ${Z}_{tj}$ \eqref{defTEXPloss} and the boosting algorithm, \tadaboost. The principle is simple and fits in four steps: (i) show that a DT can equivalently perform simple linear classifications, (ii) use a weak learner that designs splits and the boosting algorithm to fit the leveraging coefficient and compute those in closed form, (iii) simplify the expression of the loss using those, (iv) show that the expression simplified is, in fact, a CPE loss. To get (i), we remark that a DT contains a tree (graph). One can associate to each node a real value. To classify an observation, we sum all reals from the root to a leaf and decide on the class based on the sign of the prediction, just like for any real-valued predictor. Suppose we are at a leaf. What kind of weak hypotheses can create splits "in disguise"? Those can be of the form
\begin{eqnarray*}
h_j(\ve{x}) & \defeq & \iver{x_k \geq a_j} \cdot b_j, \quad a_j, b_j \in \mathbb{R}, 
\end{eqnarray*}
where the observation variable $x_k$ is assumed real valued for simplicity and the test $\iver{x_k \geq a_j}$ splits the leaf's domain in two non-empty subsets. This creates half of the split. $\overline{h}_j(\ve{x}) \defeq \iver{x_k < a_j} \cdot -b_j$ creates the other half of the split. Interestingly, $h_j$ satisfies the weak learning assumption iff $\overline{h}_j$ does \cite{mnwRC}. So we get the split design part of (ii). We compute the leveraging coefficients at the new leaves from the surrogate's minimization / boosting algorithm, end up with new real predictions at the new leaves (instead of the original $b_j, -b_j$), push those predictions in the surrogate loss for (iii), simplify it and, quite remarkably end up with a loss of the form $\expect_\leaf [\mathrm{L}(p_\leaf)]$, where $\mathrm{L}$ turns out to be the pointwise Bayes risk $\poibayesrisk$ of a proper loss \cite{mnwRC}. 

In the case of \cite{mnwRC}, it is, in fact, granted that we end up with such a "nice" CPE loss because of the choice of the surrogates at the start. In our case, however, nothing grants this \textit{a priori} if we start from the tempered exponential loss $\temperedexploss{t}$ \eqref{defTEXPloss} so it is legitimate to wonder whether such a chain of derivations (summarized) can happen to reverse engineer an interesting CPE loss: 
\begin{eqnarray}
\temperedexploss{t} \stackrel{?}{\mapsto} \mathrm{L} \stackrel{?}{\mapsto} \bayestrisk \stackrel{?}{\mapsto} \partialtloss{1}; \partialtloss{-1} \quad (\mbox{proper ? strictly proper ? for which $t$s ?, ...})
  \end{eqnarray}
  When such a complete derivation happens until the partial losses $\partialloss{1}; \partialloss{-1}$ and their properties, we shall write that minimizing $\temperedexploss{t}$ \textit{elicits} the corresponding loss and partial losses.
  \begin{theorem}\label{th-CPE}
    Minimizing $\temperedexploss{t}$ elicits the CPE loss we define as the \textbf{tempered loss}, with partial losses:
    \begin{eqnarray}
\partialtloss{1}(u) \defeq \left(\frac{1-u}{M_{1-t}(u, 1-u)}\right)^{2-t} & , & \partialtloss{-1}(u) \defeq \partialtloss{1}(1-u), \quad (t \in [-\infty,2]).
    \end{eqnarray}
    The tempered loss is symmetric, differentiable, strictly proper for $t\in (-\infty, 2)$ and proper for $t=-\infty$.
  \end{theorem}
  Differentiability means the partial losses are differentiable, and symmetry follows from the relationship between partial losses \cite{nnOT} (the proof, in \supplement, Section \ref{proof-th-CPE}, derives the infinite case, $\partialinftloss{1}(u) = 2 \cdot \iver{u \leq 1/2}$). Let us explicit the Bayes risk of the tempered loss and a key property.
  \begin{lemma}\label{lemCONT}
The Bayes risk of the tempered loss is ($M_q$ defined in Theorem \ref{thBOOST1}):
  \begin{eqnarray}
\bayestrisk(u) & = & \frac{2u(1-u)}{M_{1-t}(u, 1-u)},\label{pcbr}
  \end{eqnarray}
  and it satisfies $\forall u \in [0,1], \forall z \in [2 \cdot \min\{u,1-u\}, 1]$, $\exists t\in [-\infty, 2]$ such that $\bayestrisk(u) = z$.
\end{lemma}
Lemma \ref{lemCONT}, whose proof is trivial, allows to show a key boosting result: $t=1$ retrieves Matusita's loss, for which a near optimal boosting rate is known \cite{kmOT} while $t = -\infty$ retrieves the empirical risk, which yields the worst possible guarantee \cite{kmOT}. In between, we have, for example, CART's Gini criterion for $t=0$, which yields an intermediate boosting guarantee. Continuity with respect to $t$ of the Bayes risks in between the empirical risk and Matusita's loss means the boosting ranges of the tempered loss cover \textit{the full known spectrum of boosting rates} for $t\in [-\infty,1]$. We know of no (differentiable and) proper CPE loss with such a coverage. Note that (i) this is a non-constructive result as we do not associate a specific $t$ for a specific rate, and (ii) the state of the art boosting rates for DT induction does not seem to cover the case $t\in(1,2)$, thus left as an open question.

  \setlength\tabcolsep{0pt}
  
\begin{table*} 
  \centering
  \includegraphics[trim=0bp 0bp 0bp 0bp,clip,width=\columnwidth]{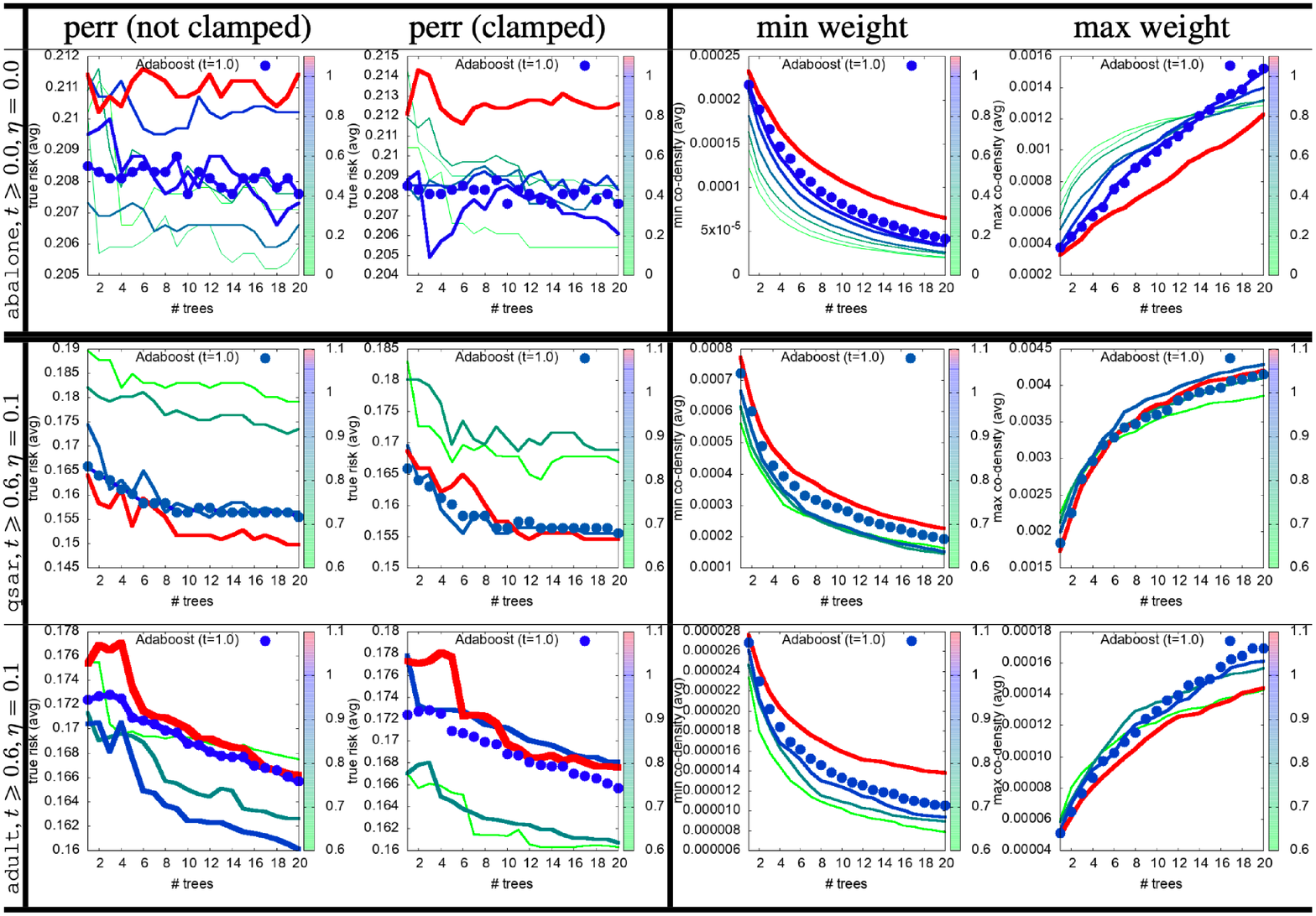}
\caption{Experiments on \tadaboost~comparing with \adaboost~($t=1$, bullets) on three domains (rows), displaying from left to right the estimated true error of non-clamped and clamped models, and the min and max codensity weights. These domains were chosen to give an example of three different situations: small values of $t$ perform well (\domainname{abalone}), the best performance is achieved by the largest $t$ (\textit{e.g.} \adaboost, \domainname{qsar}), and the worst performance is achieved by the largest $t$ (\domainname{adult}). Topmost row is without noise ($\eta = 0$) while the others are with $10\%$ training noise; $t$ scale displayed with varying color and width (colormap indicated on each plot). Averages shown for readability: see Table \ref{tab:stattests} for exhaustive statistical tests.}
    \label{tab:plots}
  \end{table*}

\section{Experiments}\label{sec-exp}

We have performed experiments on a testbed of 10 UCI domains, whose details are given in \supplement~(Section \ref{t-s-domains}). Experiments were carried out using a 10-fold stratified cross-validation procedure.

To compare \tadaboost~with \adaboost, we ran \tadaboost~with a first range of values of $t\in \{0.0, 0.2, 0.4, 0.6, 0.8, 0.9\}$. This is in the range of values covered by our convergence result for linear separators in Theorem \ref{thBOOST1}. Our results on decision tree induction cover a much wider range, in particular for $t \in (1,2)$. To assess whether this can be an interesting range to study, we added $t=1.1$ to the set of tested $t$ values. When $t>1$, some extra care is to be put into computations because the weight update becomes unbounded, in a way that is worse than \adaboost. Indeed, as can be seen from \eqref{temperedweights}, if $\mu_j  y_i h_j(\ve{x}_i) \leq -1/(t-1)$ (the example is badly classified by the current weak hypothesis, assuming wlog $\mu_j > 0$), the weight becomes infinity before renormalization. In our experiments, picking a value of $t$ close to $2$ clearly shows this problem, so to be able to still explore whether $t>1$ can be useful, we picked a value close to $1$, namely $t=1.1$, and checked that in our experiments this produced no such numerical issue. We also considered training clamped and not clamped models. 

All boosting models were trained for a number of $J=20$ decision trees. Each decision tree is induced using the tempered loss with the corresponding value of $t$ (see Theorem \ref{th-CPE}) following the classical top-down template, which consists in growing the current heaviest leaf in the tree and picking the best split for the leaf chosen. We implemented \tadaboost~exactly as in Section \ref{sec-defs-tada}, including computing leveraging coefficients as suggested. Thus, we do not scale models. More details are provided in \supplement.
In our experiments, we also included experiments on a phenomenon highlighted more than a decade ago \cite{lsRC} and fine-tuned more recently \cite{mnwRC}, the fact that a convex booster's model is the weakest link when it has to deal with noise in training data. This is an important task because while the tempered exponential loss is convex, it does not fit into the blueprint loss of \cite[Definition 1]{lsRC} because it is not $C^1$ if $t\neq 1$. One might thus wonder how \tadaboost~behaves when training data is affected by noise. Letting $\eta$ denote the proportion of noisy data in the training sample, we tried $\eta \in \{0.0, 0.1\}$. We follow the noise model of \cite{lsRC} and thus independently flip the true label of each example with probability $\eta$.

For each run, we recorded the average test error and the average maximum and minimum co-density weight. Table \ref{tab:plots} presents a subset of the results obtained on three domains. Table \ref{tab:stattests} presents a more synthetic view in terms of statistical significance of the results for $t\neq 1$ vs. $t=1$ (\adaboost). The table reports only results for $t\geq 0.6$ for synthesis purposes. Values $t<0.6$ performed on average slightly worse than the others \textit{but} on some domains, as the example of \domainname{abalone} suggests in Table \ref{tab:stattests} (the plots include all values of $t$ tested in $[0,1.1]$), we clearly got above-par results for such small values of $t$, both in terms of final test error but also fast early convergence to low test error. This comment can be generalized to all values of $t$.

The weights reveal interesting patterns as well. First, perhaps surprisingly, we never encountered the case where weights switch off, regardless of the value of $t$. The average minimum weight curves of Table \ref{tab:plots} generalize to all our tests (see \supplement). This does not rule out the fact that boosting for a much longer number of iterations might lead to weights switching off/on, but the fact that this does not happen at least early during boosting probably comes from the fact that the leveraging coefficients for weights ($\mu_.$) are bounded. Furthermore, their maximal absolute value is all the smaller as $t$ decreases to $0$. Second, there is a pattern that also repeats on the maximum weights, not on all domains but on a large majority of them and can be seen in \domainname{abalone} and \domainname{adult} in Table \ref{tab:plots}: the maximum weight of \adaboost~tends to increase much more rapidly compared to \tadaboost~with $t<1$. In the latter case, we almost systematically observe that the maximum weight tends to be upperbounded, which is not the case for \adaboost~(the growth of the maximal weight looks almost linear). Having bounded weights could be of help to handle numerical issues of (ada)boosting \cite{kIA}.

Our experiments certainly confirm the boosting nature of \tadaboost~if we compare its convergence to that of \adaboost: more often than not, it is in fact comparable to that of \adaboost. While this applies broadly for $t\geq 0.6$, we observed examples where much smaller values (even $t=0.0$) could yield such fast convergence. Importantly, this applies to clamped models as well and it is important to notice because it means attaining a low "boosted" error does not come at the price of learning models with large range. This is an interesting property: for $t=0.0$, we would be guaranteed that the computation of the clamped prediction is always in $[-1,1]$. Generalizing our comment on small values of $t$ above, we observed that an efficient tuning algorithm for $t$ could be able to get very substantial leverage over \adaboost. Table \ref{tab:stattests} was crafted for a standard limit $p$-val of 0.1 and "blurs" the best results that can be obtained. On several domains (\domainname{winered}, \domainname{abalone}, \domainname{eeg}, \domainname{creditcard}, \domainname{adult}), applicable $p$-values for which we would conclude that some $t\neq 1$ performs better than $t=1$ drop in between $7E-4$ and $0.05$. Unsurprisingly, \adaboost~also manages to beat significantly alternative values of $t$ in several cases. Our experiments with training noise ($\eta = 0.1$) go in the same direction. Looking at Table \ref{tab:plots}, one could eventually be tempted to conclude that $t$ slightly smaller than 1.0 may be a better choice than adaboosting ($t=1$), as suggested by our results for $t=0.9$, but we do not think this produces a general "rule-of-thumb". There is also no apparent "noise-dependent" pattern that would obviously separate the cases $t<1$ from $t=1$ even when the tempered exponential loss does not fit to \cite{lsRC}'s theory. Finally, looking at the results for $t>1$ also yields the same basic conclusions, which suggests that boosting can be attainable outside the range covered by our theory (in particular Theorem \ref{thBOOST1}).

All this brings us to the experimental conclusion that the question does not reside on opposing the case $t\neq 1$ to the case $t=1$. Rather, our experiments suggest -- pretty much like our theory does -- that the actual question resides in how to efficiently \textit{learn} $t$ on a domain-dependent basis. Our experiments indeed demonstrate that substantial gains could be obtained, but our simple attempts at performing such approaches, motivated by the following Section, did not produce a statistically impactful candidate so far.

  \setlength\tabcolsep{6pt}
  
\begin{table*}
  \centering
  \resizebox{\columnwidth}{!}{%
  \begin{tabular}{r?c|c|c|c|c|c|c|c?c|c|c|c|c|c|c|c}\Xhline{2pt}
    $\eta$ & \multicolumn{8}{c?}{$0.0$} & \multicolumn{8}{c}{$0.1$}\\
$t$    & \multicolumn{2}{c|}{$0.6$} & \multicolumn{2}{c|}{$0.8$} & \multicolumn{2}{c|}{$0.9$} & \multicolumn{2}{c?}{$1.1$}& \multicolumn{2}{c|}{$0.6$} & \multicolumn{2}{c|}{$0.8$} & \multicolumn{2}{c|}{$0.9$} & \multicolumn{2}{c}{$1.1$}\\
    $\iver{\mbox{clamped}}$ & 0 & 1  & 0 & 1  & 0 & 1  & 0 & 1 & 0 & 1  & 0 & 1  & 0 & 1  & 0 & 1 \\ \Xhline{2pt}
    $\#$better &  2 & 3  &  1 &  2 & 1  & 3  &  &  &  1 & 1 & 1 & 2 & 2 & 1 &  & \\\hline
    $\#$equivalent & 5 & 5 & 6 & 6 & 7 & 7 & 6 & 7 & 4 & 8 & 8 & 7 & 8 & 9 & 8 & 10 \\\hline
    $\#$worse & 3 & 2 & 3 & 2 & 2 & & 4 & 3 & 5 & 1 & 1 & 1 & & &2  & \\\Xhline{2pt}
\end{tabular}
}
\caption{Outcomes of student paired $t$-tests over 10 UCI domains, with training noise $\eta\in \{0.0,0.1\}$, for $t \in \{0.6, 0.8, 0.9, 1.0, 1.1\}$ and with / without clamped models. For each triple ($\eta$, $t$, $\iver{\mbox{clamped}}$), we give the number of domains for which the corresponding setting of \tadaboost~is statistically better than \adaboost ($\#$better), the number for which it is statistically worse ($\#$worse) and the number for which we cannot reject the assumption of identical performances. Threshold $p-$val = 0.1.}
    \label{tab:stattests}
  \end{table*}

\section{Discussion: loss selection with \tadaboost}\label{sec-disc}

Our theory yields a family of surrogate loss functions and corresponding training algorithms for the induction of linear combinations of classifiers or the induction of decision trees, with boosting-compliant rates on training. This formal picture looks quite uniform from the training standpoint but it is nuanced experimentally by the fact that the best values for $t$ for good generalization depend on the domain at hand (and on additional experimental factors, such as the presence of noise, the size of models, etc.). While this could be expected because generalization entails prediction on unseen data, there would be a specific reason, in our case, not just to have $t$ domain dependent, but in fact to tune $t$ during training \cite{nnOT}. As we explain, this problem entails questions relative to training and generalisation and to make a parallel with model selection in ML \cite[Chapter 4]{mrtFO}, we call it loss selection.

Suppose that $\mathcal{S}$ is sampled i.i.d. according to some unknown $\mathcal{D}$. Let $\mathcal{H}$ denote a set of linear combinations of decision trees. We know that with probability $\geq 1-\delta$, every $H \in \mathcal{H}$ has
\begin{eqnarray*}
\temperedexploss{t}(H, \mathcal{D}) \leq \temperedexploss{t}(H, \mathcal{S}) + \mathcal{O}(L_t R_m(\mathcal{H})) + Q(m,\delta),\label{eqGEN}
\end{eqnarray*}
where $L_t$ is the Lipschitz constant of the tempered exponential loss, $R_m(\mathcal{H})$ is the Rademacher complexity of $\mathcal{H}$ (a capacity parameter) and $Q$ does not depend on $t$ nor $H$, see \cite{bmRA} and \cite{mrtFO}. Since $\zoloss(H, \mathcal{D})\leq \temperedexploss{t}(H, \mathcal{D})$, \eqref{eqGEN} also brings guarantees on the true risk of every classifier in $\mathcal{H}$. Reasoning in terms of structural risk minimization through parameter $t$, we should pick $t$ not just to get a convenient loss to minimize empirically ($\temperedexploss{t}$)
but also to reduce the uncertainty on translating good empirical results in generalization as well, which would command to reduce $L_t$ as well. We assume that the domain in which we are allowed to fix $t$ is $[0,1]$.

Denote for short $u_t(z) \defeq \exp_t^{2-t}(-z)$. Without restriction on $z$, $u_t$ is not Lipschitz, but (i) our experiments display that we tend to learn classifiers with bounded magnitude and (ii) scaling a classifier by a positive factor does not change its empirical risk, so assume its argument satisfies $|z| \leq Z$ for some $Z>0$. In this case, the Lipschitz constant can be computed as $L_{t,Z} = (2-t) \exp_t(Z)$, which brings
\begin{eqnarray*}
  \frac{\partial L_{t,Z}}{\partial t} & = & \frac{\exp_t^t(Z)}{1-t^*}\cdot \left(\exp_t^{1-t}(Z)\cdot \left[\log\exp_t(Z)- (1-t^*)\right] -Z \right) \: \left(\mbox{recall } t^* \defeq \frac{1}{2-t}\right).
\end{eqnarray*}
We observe $\frac{\partial L_{t,0}}{\partial t} = -1$ but as $Z$ increases, the function can get $>0$: for example, if $Z=1/(1-t)$, the function is positive for $t\in [1/2, 1]$. From a ML standpoint, we thus have the following behaviour impacting the Lipschitz constant:
\begin{itemize}
\item [{(\textbf{L})}] early during training or when we have a small number of trees, we can have $Z$ very small, while later on, as the number of trees increases, the relevant $Z$ increases as well.
\end{itemize}
Hence, to mitigate the impact of the capacity parameter in \eqref{eqGEN}, {(\textbf{L})} suggests to increase $t$ early during training while eventually reducing it later on. What about the loss that we minimize during training, \eqref{defTEXPloss} ? It can be shown that
\begin{eqnarray}
\frac{\partial {Z}^{2-t}_{tj}}{\partial t} &= & \frac{1}{|1-t|} \sum_{i \in [m]} \left( \frac{1}{2-t}\cdot q^{2-t}_{(j+1)i} \cdot \log q^{2-t}_{(j+1)i} - q_{(j+1)i}q^{1-t}_{ji}\cdot \log q^{2-t}_{ji} \right),
\end{eqnarray}
which is perhaps not easily readable, but when $\ve{q}_j > \ve{0}$, has a nice equivalent expression. For $t \in [0,1)$, let $H_t(z) \defeq (1/(2-t)) \cdot \left(z^{2-t} \cdot \log z^{2-t} - z^{2-t}\right)$. 
  \begin{lemma}
    Suppose $\ve{q}_j > \ve{0}$. Recalling $D_.(.\|.)$ is the notation for a Bregman divergence, we have
    \begin{eqnarray}
\frac{\partial {Z}^{2-t}_{tj}}{\partial t} &= & \frac{1}{|1-t|} \cdot D_{H_t} (\ve{q}_{j+1}\| \ve{q}_j) + H_t(\ve{q}_j). \label{eqDERPARZ}
    \end{eqnarray}
  \end{lemma}
  The arguments of $H_t$ are \acrotem s, so \eqref{eqDERPARZ} is in disguise (and up to multiplicative factors) the subtraction of two non-negative quantities: the Kullback-Leibler divergence between two successive co-densities and Shannon's entropy of the previous co-density. The total is thus negative especially when $\ve{q}_{j+1}$ and $\ve{q}_j$ come closer to each other (a Bregman divergence satisfies the identity of indiscernibles). From a ML standpoint, we thus have the following behaviour impacting the loss:
  \begin{itemize}
  \item [{(\textbf{Z})}] early during training or as long as \tadaboost~is far from the training optimum of ${Z}^{2-t}_{tj}$, $\ve{q}_{j+1}$ and $\ve{q}_j$ are substantially different from each other while as \tadaboost~gets closer to the optimal classifier on training, $\ve{q}_{j+1} \rightarrow \ve{q}_j$.
  \end{itemize}
  Hence, increasing $t$ early during training or reducing it later on, as suggested by (\textbf{L}) may lead from {(\textbf{Z})} to a worse upperbound of the empirical risk in \eqref{eqGEN}. Also true is the fact that reducing the loss parameter via {(\textbf{Z})} may worsen the dependency on the Lipschitz constant from (\textbf{L}). There should thus be data- and classifier-dependent ways to tune $t$. We have made simple attempts at loss selection but none has proven impactful in terms of results.

\section{Conclusion}\label{sec-conc}

%Lipschitzness with clamped classifiers ?

\adaboost~is one of the original and simplest Boosting algorithms.
In this paper, we generalized \adaboost~to maintaining a
tempered measure over the examples by minimizing a tempered relative entropy.
We kept the setup as simple as possible and therefore
focused on generalizing \adaboost.
However, more advanced boosting algorithms have been designed based on
relative entropy minimization subject to linear constraints.
There are versions that constrain the edges of all past hypotheses to be zero \cite{totcorr}.
Also, when the maximum margin of the game is larger than zero, then \adaboost~cycles over
non-optimal solutions \cite{cycles}. Later Boosting algorithms provably
optimize the margin of the solution by adjusting the
constraint value on the dual edge away from zero (see, e.g.,
\cite{adastar}).
Finally, the ELRP-Boost algorithm optimizes a trade-off
between relative entropy and the edge \cite{erlp}.
We conjecture that all of these orthogonal
direction have generalizations to the tempered case as well
and are worth exploring.

These are theoretical directions that, if successful, would contribute
to bringing more tools to the design of rigorous boosting
algorithms. This is important because boosting suffers several
impediments, not all of which we have mentioned: for example, to get
statistical consistency for \adaboost, it is known that early stopping
is mandatory \cite{btAI}. More generally, non-Lipschitz losses like
the exponential loss seem to be harder to handle compared to Lipschitz
losses \cite{tBW} (but they yield, in general, better convergence rates on training). The validity of the weak learning assumption of boosting can also be discussed, in particular, regarding the negative result of \cite{lsRC} which advocates, beyond just better (ada)boosting, for boosting for \textit{more} classes of models / architectures \cite{mnwRC}. Alongside this direction, we feel that our experiments on noise handling give a preliminary account of the fact that there is no "one $t$ fits all" case, but a much more in-depth analysis is required to elicit / tune a "good" $t$. This is a crucial issue for noise handling \cite{mnwRC}, but as we explain in Section \ref{sec-exp}, this could bring benefits in much wider contexts as well. Putting this in the context of loss selection (Section \ref{sec-disc}) suggests potential directions to solve the problem. We tried a few elementary solutions alongside those directions but did not get anything substantial in terms of results. this is another direction worth exploring.

% \section{Broader impact, discussion and conclusion}\label{sec-conc}

% %Lipschitzness with clamped classifiers ?

% AdaBoost is one of the original and simplest Boosting algorithms.
% It updates its distribution by minimizing a relative entropy
% to the last distribution subject to the constraint that edge of the last
% hypothesis become zero after the update \cite{kwBA}.
% In this paper we generalized AdaBoost to maintaining a
% tempered measure over the examples by minimizing a tempered relative entropy instead.
% We kept the setup as simple as possible and therefore
% focused on generalizing AdaBoost.
% However more advanced boosting algorithms have been designed based on
% relative entropy minimization subject to linear constraints.
% There are versions that constrain the edges of all past hypotheses to be zero \cite{totcorr}.
% Also, when the maximum margin of the game is larger than zero, then AdaBoost cycles over
% non-optimal solutions \cite{cycles}. Later Boosting algorithms provably
% optimize the margin of the solution by adjusting the
% constraint value on the dual edge away from zero (see e.g.
% \cite{adastar}).
% Finally, the ELRP-Boost algorithm optimizes a trade off
% between relative entropy and the edge \cite{erlp}.
% We conjecture that all of these orthogonal
% direction have generalizations to the tempered case as well
% and are worth exploring.

%\begin{ack} 
%\end{ack}

\bibliographystyle{plain}
\bibliography{bibgen}

\newpage
\appendix
\onecolumn
\renewcommand\thesection{\Roman{section}}
\renewcommand\thesubsection{\thesection.\arabic{subsection}}
\renewcommand\thesubsubsection{\thesection.\thesubsection.\arabic{subsubsection}}

\renewcommand*{\thetheorem}{\Alph{theorem}}
\renewcommand*{\thelemma}{\Alph{lemma}}
\renewcommand*{\thecorollary}{\Alph{corollary}}

\renewcommand{\thetable}{A\arabic{table}}

\begin{center}
\Huge{Appendix}
\end{center}

To
differentiate with the numberings in the main file, the numbering of
Theorems, etc. is letter-based (A, B, ...).

\section*{Table of contents}

\noindent \textbf{A short primer on Tempered Exponential Measures} \hrulefill Pg \pageref{sec-sup-primer}\\

\noindent \textbf{Supplementary material on proofs} \hrulefill Pg
\pageref{sec-sup-pro}\\
\noindent $\hookrightarrow$ Proof of Theorem \ref{thENTBOOST} \hrulefill Pg \pageref{proof-thENTBOOST}\\
\noindent $\hookrightarrow$ Proof of Theorem \ref{thBOOST1} \hrulefill Pg \pageref{proof-thBOOST1}\\
\noindent $\hookrightarrow$ Proof of Theorem \ref{th-CPE} \hrulefill Pg \pageref{proof-th-CPE}\\

\noindent \textbf{Supplementary material on experiments} \hrulefill Pg
\pageref{sec-sup-exp}\\
\noindent $\hookrightarrow$ Domains \hrulefill Pg \pageref{sec-doms}\\
\noindent $\hookrightarrow$ Implementation details and full set of experiments on linear combinations of decision trees \hrulefill Pg \pageref{sec-full-exp}

\newpage %lem-split-gives-wla

\section{A short primer on Tempered Exponential Measures} \label{sec-sup-primer}

We describe here the minimal amount of material necessary to understand how our approach to boosting connects to these measures. We refer to \cite{anwCA} for more details. With a slight abuse of notation, we define the perspective transforms $(\log_t)^*(z) \defeq t^* \cdot \log_{t^*} (z/t^*)$ and $(\exp_t)^*(z) \defeq t^* \cdot \exp_{t^*} (z/t^*)$. Recall that $t^* \defeq 1/(2-t)$.
\begin{definition}\cite{anwCA}
  A tempered exponential measure (\acrotem) family is a set of unnormalized densities in which each element admits the following canonical expression:
  \begin{eqnarray}
\label{eq:exp_t_density_form}
q_{t|\ve{\theta}} (\ve{x}) \defeq \frac{\exp_t(\ve{\theta}^\top \ve{\phi}(\ve{x}))}{\exp_t(G_t(\ve{\theta}))} = \exp_t(\ve{\theta}^\top \ve{\phi}(\ve{x}) \ominus_t G_t(\ve{\theta})) \quad\left(a \ominus_t b \defeq \frac{a-b}{1+(1-t)b} \right),
  \end{eqnarray}
  where $\ve{\theta}$ is the element's natural parameter, $\ve{\phi}(\ve{x})$ is the sufficient statistics and
  \begin{eqnarray*}
    G_t(\ve{\theta}) & = & (\log_t)^* \int (\exp_t)^* (\ve{\theta}^\top \ve{\phi}(\ve{x}))\mathrm{d}\xi
  \end{eqnarray*}
  is the (convex) cumulant, $\xi$ being a base measure (implicit).
\end{definition}
Except for $t=1$ (which reduces a \acrotem~family to a classical exponential family), the total mass of a \acrotem~is not 1 (but it has an elegant closed form expression \cite{anwCA}). However, the exponentiated $q_{t|\ve{\theta}}^{1/t^*}$ does sum to 1. In the discrete case, this justifies extending the classical simplex to what we denote as the co-simplex.
\begin{definition}\label{defCOSIMPLEX}
The co-simplex of $\mathbb{R}^m$, $\tilde{\Delta}_m$ is defined as $\tilde{\Delta}_m \defeq \{\ve{q}\in \mathbb{R}^m: \ve{q} \geq \ve{0} \wedge \ve{1}^\top \ve{q}^{1/t^*} = 1\}$.
\end{definition}
The connection between \tadaboost's update and \acrotem's is immediate from the equation's update (\eqref{defWoptM} in \mainfile). We can show that $\tilde{\Delta}_m$ can also be represented as \acrotem s.
\begin{lemma}\label{lemDeltam}
$\tilde{\Delta}_m$ is a (discrete) family of tempered exponential measures.
\end{lemma}
\begin{proof}
We proceed as in \cite[Section 2.2.2]{aIG} for exponential families: let ${\ve{q}} \in \tilde{\Delta}_m$, which we write
\begin{eqnarray}
  {q}(n) & \defeq & \sum_{i\in [m]} {q}_i \cdot \iver{i=n}, n \in [m].\label{defPN}
\end{eqnarray}
$\iver{\pi}$, the Iverson bracket \cite{kTN}, takes value 1 if Boolean predicate $\pi$ is true (and 0 otherwise). We create $m-1$ natural parameters and the cumulant,
\begin{eqnarray*}
\theta_i \defeq \log_t \frac{{q}_i}{{q}_m}, i \in [m-1] & ; & G_t(\ve{\theta}) \defeq \log_t \frac{1}{{q}_m},
\end{eqnarray*}
and end up with \eqref{defPN} also matching the atom mass function
\begin{eqnarray*}
{q}(n) & = & \frac{\exp_t\left(\sum_{i\in [m-1]} \theta_i \cdot \iver{i=n}\right)}{\exp_t G_t(\ve{\theta})},
\end{eqnarray*}
which clearly defines a tempered exponential measure over $[m]$. This ends the proof of Lemma \ref{lemDeltam}.
  \end{proof}

\newpage
  
\section{Supplementary material on proofs} \label{sec-sup-pro}

\subsection{Proof of Theorem  \ref{thENTBOOST}}\label{proof-thENTBOOST}

To improve readability, we remove dependency in $t$ in normalization coefficient $Z$. We use notations from \cite[proof of Theorem 3.2]{anwCA} and denote the Lagrangian
\begin{eqnarray}
\mathcal{L} & = & \Delta(\tilde{\ve{q}} \| \ve{q}) + \lambda \left(\sum_i \tilde{q}^{1/t^*}_{i}-1\right) - \sum_i \nu_i \tilde{q}_{i} + \mu \sum_i \tilde{q}_{i} u_i,
\end{eqnarray}
which yields $\partial \mathcal{L} / \partial \tilde{q}_i = \log_t \tilde{q}_i  - \log_t q_{i} + \lambda \tilde{q}^{1-t}_{i} - \nu_i + \mu u_i$ ($\lambda$ absorbs factor $2-t$), and, rearranging (absorbing factor $1-t$ in $\nu_i$),
\begin{eqnarray}
(1+(1-t)\lambda) \tilde{q}^{1-t}_i & = & \nu_i + 1 + (1-t) (\log_t q_{i} - \mu  u_i), \forall i \in [m]. \label{lagAN}
\end{eqnarray}
We see that $\lambda \neq -1/(1-t)$ otherwise the Lagrangian drops its dependence in the unknown. In fact, the solution necessarily has $1+(1-t)\lambda > 0$. To see this, we distinguish two cases: (i) if some $u_k = 0$, then since $\log_t q_{k} \geq -1/(1-t)$ there would be no solution to \eqref{lagAN} if $1+(1-t)\lambda < 0$ because of the KKT conditions $\nu_i \geq 0, \forall i\in [m]$;  (ii) otherwise, if all $u_k \neq 0, \forall k \in [m]$, then there must be two coordinates of different signs otherwise there is no solution to our problem \eqref{defBATEP} (main file, we must have indeed $\tilde{\ve{q}} \geq 0$ because of the co-simplex constraint). Thus, there exists at least one coordinate $k\in [m]$ for which $- (1-t) \mu  u_k > 0$ and since $\log_t q_{k} \geq -1/(1-t)$ (definition of $\log_t$) and $\nu_k \geq 0$ (KKT conditions), the RHS of \eqref{lagAN} for $i=k$ is $>0$, preventing $1+(1-t)\lambda < 0$ in the LHS.

We thus have $1+(1-t)\lambda > 0$. The KKT conditions $(\nu_i \geq 0, \nu_i \tilde{q}_{i} = 0, \forall i\in [m])$ yield the following: $1 + (1-t) (\log_t q_{i} - \mu  u_i) > 0$ imply $\nu_i = 0$ and $1 + (1-t) (\log_t q_{i} - \mu  u_i) \leq 0$ imply $\tilde{q}^{1-t}_i = 0$ so we get the necessary form for the optimum:
\begin{eqnarray}
\tilde{q}_i & = & \frac{\exp_t\left(\log_t q_{i} - \mu
                  u_i\right)}{\exp_t \lambda} \nonumber\\
  & = & \frac{q_{i} \otimes_t \exp_t(- \mu  u_i)}{Z_t}\label{defWopt},
\end{eqnarray}
where $\lambda$ or $Z_t \defeq \exp_t \lambda$ ensures normalisation for the co-density. Note that we have a simplified expression for the co-density: 
\begin{eqnarray}
p_i & = & \frac{p_{ji} \otimes_{t^*} \exp_{t^*}(- \mu  u_i/t^*)}{Z^{\mbox{\tiny{co}}}_t},\label{defWCODopt}
\end{eqnarray}
with $Z^{\mbox{\tiny{co}}}_t \defeq Z^{1/t^*}_t = \sum_i p_{ji} \otimes_{t^*} \exp_{t^*}(- \mu  u_i/t^*)$. For the analytic form in \eqref{defWopt}, we can simplify the Lagrangian to a dual form that depends on $\mu$ solely:
\begin{eqnarray}
\mathcal{D}(\mu) & = & \Delta(\tilde{\ve{q}}(\mu) \| \ve{q}) + \mu \sum_i \tilde{q}_{i}(\mu) u_i.
\end{eqnarray}
The proof of \eqref{defLossFromZ} (main file) is based on a key Lemma.
\begin{lemma}\label{lemLogZ}
For any $\tilde{\ve{q}}$ having form \eqref{defWopt} such that $\tilde{\ve{q}}^\top \ve{u} = 0$, $\mathcal{D}(\mu) = -\log_t Z_t(\mu)$.
\end{lemma}
\begin{proof}
  For any $\tilde{\ve{q}}$ having form \eqref{defWopt}, denote
  \begin{eqnarray}
    [m]_* & \defeq & \{i : \tilde{q}_i \neq 0\}.\label{defMSTAR}
  \end{eqnarray}
  We first compute (still using $\lambda \defeq \log_t Z_t(\mu)$ for short):
\begin{eqnarray}
  A & \defeq & \sum_i \tilde{q}_i \cdot \log_t \tilde{q}_i\nonumber\\
    & = & \sum_{i\in [m]_*} \tilde{q}_i \cdot \log_t\left( \frac{\exp_t\left(\log_t q_{i} - \mu  u_i\right)}{\exp_t \lambda}\right)\nonumber\\
    & = & \sum_{i\in [m]_*} \tilde{q}_i \cdot \left(\frac{1}{1-t}\cdot\left[\frac{1+(1-t)(\log_t q_{i} - \mu  u_i)}{1+(1-t)\lambda}-1\right]\right)\nonumber\\
    & = & \frac{1}{1-t} \cdot \sum_{i\in [m]_*}\tilde{q}_i \cdot \left(\frac{q_{i}^{1-t} - (1-t) \mu u_i}{1+(1-t)\lambda} \right) - \frac{1}{1-t} \cdot \sum_{i\in [m]_*}\tilde{q}_i \nonumber\\
    & = & -\frac{\mu}{1+(1-t)\lambda} \cdot  \sum_{i\in [m]_*}\tilde{q}_i u_i + \frac{1}{(1-t)(1+(1-t)\lambda)} \cdot \sum_{i\in [m]_*}\tilde{q}_i q_{i}^{1-t} - \frac{1}{1-t} \cdot \sum_{i\in [m]_*}\tilde{q}_i\nonumber\\
  & = & \underbrace{-\frac{\mu}{1+(1-t)\lambda} \cdot \tilde{\ve{q}}^\top \ve{u}}_{\defeq B} + \underbrace{\frac{1}{(1-t)(1+(1-t)\lambda)} \cdot \sum_{i\in [m]}\tilde{q}_i q_{i}^{1-t}}_{\defeq C} - \underbrace{\frac{1}{1-t} \cdot \sum_{i\in [m]}\tilde{q}_i}_{\defeq D} \label{simplA}.
\end{eqnarray}
  Remark that in the last identity, we have put back summations over the complete set $[m]$ of indices. We note that $B=0$ because $\tilde{\ve{q}}^\top \ve{u} = 0$. We then remark that without replacing the expression of $\tilde{\ve{q}}$, we have in general for any $\tilde{\ve{q}}\in \tilde{\Delta}_m$:
\begin{eqnarray*}
  E & \defeq & \sum_{i\in[m]} \tilde{q}_i \cdot (\log_t \tilde{q}_i - \log_t q_{i})\\
    & = & \sum_{i\in[m]} \tilde{q}_i \cdot \left(\frac{1}{1-t}\cdot\left(\tilde{q}_i^{1-t} - 1\right)-\frac{1}{1-t}\cdot\left(q_{i}^{1-t} - 1\right)\right)\\
    & = & \frac{1}{1-t} \cdot \sum_{i\in[m]}\tilde{q}^{2-t}_i - \frac{1}{1-t} \cdot \sum_{i\in[m]}\tilde{q}_i q_{i}^{1-t}\\
  & = & \frac{1}{1-t} \cdot\left(1-\sum_{i\in[m]}\tilde{q}_i q_{i}^{1-t}\right),
\end{eqnarray*}
and we can check that for any $\tilde{\ve{q}}, \ve{q} \in \tilde{\Delta}_m$, $E = \Delta(\tilde{\ve{q}} \| \ve{q})$. We then develop $\Delta(\tilde{\ve{q}} \| \ve{q})$ with a partial replacement of $\tilde{\ve{q}}$ by its expression:
\begin{eqnarray*}
  \Delta(\tilde{\ve{q}} \| \ve{q}) & = & A - \sum_i \tilde{q}_i \log_t q_{i}\\
                                             & = & A - \frac{1}{1-t} \cdot \sum_i \tilde{q}_i q_{i}^{1-t} + \frac{1}{1-t} \cdot \sum_i \tilde{q}_i \\
                                             & = & C - \frac{1}{1-t} \cdot \sum_i \tilde{q}_i q_{i}^{1-t} \\
                                             & = & \frac{1}{1-t} \cdot \left(\frac{1}{1+(1-t)\lambda} - 1\right) \cdot \sum_i \tilde{q}_i q_{i}^{1-t} \\
                                             & = & - \frac{\lambda}{1+(1-t)\lambda} \cdot \sum_i \tilde{q}_i q_{i}^{1-t} \\
  & = &  - \frac{\lambda}{1+(1-t)\lambda} \cdot \left(1 - (1-t) \cdot \Delta(\tilde{\ve{q}} \| \ve{q})\right).
\end{eqnarray*}
Rearranging gives that for any $\tilde{\ve{q}}, \ve{q} \in \tilde{\Delta}_m$ such that (i) $\tilde{\ve{q}}$ has the form \eqref{defWopt} for some $\mu \in \mathbb{R}$ and (ii) $\tilde{\ve{q}}^\top \ve{u} = 0$,
\begin{eqnarray*}
  \Delta(\tilde{\ve{q}} \| \ve{q}) & = & -\lambda\\
  & = & -\log_t (Z_t),
\end{eqnarray*}
as claimed. This ends the proof of Lemma \ref{lemLogZ}.
\end{proof}
\noindent We thus get from the definition of the dual that $\mu = \arg\max -\log_t Z_t(\mu) = \arg\min Z_t(\mu)$. We have the explicit form for $Z_t$:  
\begin{eqnarray*}
  Z_t(\mu) & = & \left(\sum_i \exp^{2-t}_t\left(\log_t q_{i} - \mu  u_i\right)\right)^{\frac{1}{2-t}}\\
  & = & \left(\sum_{i\in [m]_*} \exp^{2-t}_t\left(\log_t q_{i} - \mu  u_i\right)\right)^{\frac{1}{2-t}},
\end{eqnarray*}
where $[m]_*$ is defined in \eqref{defMSTAR}. We remark that the last expression is differentiable in $\mu$, and get
\begin{eqnarray}
  Z'_t(\mu) & = & \frac{1}{2-t} \cdot \left(\sum_{i\in [m]_*} \exp^{2-t}_t\left(\log_t q_{i} - \mu  u_i\right)\right)^{-\frac{1-t}{2-t}} \nonumber\\
                                             & & \cdot (2-t) \sum_{i\in [m]_*} \exp^{1-t}_t\left(\log_t q_{i} - \mu  u_i\right) \cdot \exp^{t}_t\left(\log_t q_{i} - \mu  u_i\right) \cdot - u_i\nonumber\\
                                             & = & - Z_t^{t-1} \cdot \sum_{i\in [m]_*} \exp_t\left(\log_t q_{i} - \mu  u_i\right) \cdot u_i\nonumber\\
                                             & = & - Z_t^{t} \cdot \sum_{i\in [m]_*} \tilde{q}_i u_i\nonumber\\
                                             & = & - Z_t^{t} \cdot \sum_{i\in [m]} \tilde{q}_i u_i\nonumber\\
  & = & - Z_t^{t} \cdot \tilde{\ve{q}}^\top \ve{u},\label{optZj}
\end{eqnarray}
so
\begin{eqnarray*}
  \frac{\partial -\log_t (Z_t)}{\partial \mu} & = & - Z^{-t}_t Z'_t\\
  & = & \tilde{\ve{q}}(\mu)^\top \ve{u},
\end{eqnarray*}
and we get that any critical point of $Z_t(\mu)$ satisfies $\tilde{\ve{q}}(\mu)^\top \ve{u} = 0$. A sufficient condition to have just one critical point, being the minimum sought is the strict convexity of $Z_t(\mu)$. The next Lemma provides the proof that it is for all $t>0$.
\begin{lemma}\label{convexZ}
$Z''_t(\mu) \geq t\cdot Z_t(\mu)^{2t-1} (\tilde{\ve{q}}(\mu)^\top \ve{u})^2$.
\end{lemma}
\begin{proof}
After simplifications, we have
\begin{eqnarray}
  Z_t^{3-2t} \cdot Z''_t  & = & (t-1) \cdot \left(\sum_{i\in [m]} \exp_t\left(\log_t q_{i} - \mu  u_i\right) \cdot u_i\right)^2 \\
                                          & & +\left(\sum_{i\in [m]} \exp^{2-t}_t\left(\log_t q_{i} - \mu  u_i\right)\right) \cdot \left( \sum_{i\in [m]} \exp^t_t\left(\log_t q_{i} - \mu  u_i\right) \cdot u^2_i\right) \\
  & = &  (t-1) \cdot \sum_{i, k\in [m]} Q_i Q_k u_i u_k + \sum_{i, k\in [m]} Q^{2-t}_i Q^t_k  u^2_k,\label{zConv1}
\end{eqnarray}
where we have let $Q_i \defeq \exp_t\left(\log_t q_{i} - \mu  u_i\right) \geq 0$. Since $a^2 + b^2 \geq 2ab$, we note that for any $i\neq k$,
\begin{eqnarray}
  Q^{2-t}_i Q^t_k  u^2_k +  Q^{2-t}_k Q^t_i  u^2_i & \geq & 2 \sqrt{Q^{2-t}_i Q^t_k Q^{2-t}_k Q^t_i} u_i u_k\nonumber\\
  & & = 2 Q_i Q_k u_i u_k,\label{prod1}
  \end{eqnarray}
  so we split \eqref{zConv1} in two terms and get
  \begin{eqnarray}
    Z_t^{3-2t} \cdot Z''_t  & = &  (t-1) \cdot \sum_{i \in [m]} Q^2_i u^2_i + \sum_{i \in [m]} Q^{2-t}_i Q^t_i  u^2_i\nonumber\\
                                            & & + \sum_{i, k\in [m], i<k} 2(t-1)Q_i Q_k u_i u_k + \sum_{i, k\in [m], i<k} Q^{2-t}_i Q^t_k  u^2_k +  Q^{2-t}_k Q^t_i  u^2_i \nonumber\\
                                            & = &  t \cdot \sum_{i \in [m]} Q^2_i u^2_i\nonumber\\
    & & + \sum_{i, k\in [m], i<k} 2(t-1)Q_i Q_k u_i u_k + \sum_{i, k\in [m], i<k} Q^{2-t}_i Q^t_k  u^2_k +  Q^{2-t}_k Q^t_i  u^2_i \nonumber\\
                                            & \geq &  t \cdot \sum_{i \in [m]} Q^2_i u^2_i + 2t \cdot \sum_{i, k\in [m], i<k} Q_i Q_k u_i u_k  \label{zConv2}\\
                                            & & = t\cdot \left(\sum_{i \in [m]} \exp_t\left(\log_t q_{i} - \mu  u_i\right) \cdot u_i\right)^2\nonumber\\
    & = & t Z^2_t \cdot (\tilde{\ve{q}}^\top \ve{u})^2,\label{binfZ1}
\end{eqnarray}
where we have used \eqref{prod1} in \eqref{zConv2}. Since $Z_t(\mu) > 0$, we get the statement of Lemma \ref{convexZ} after reorganising \eqref{binfZ1}.
\end{proof}
Lemma \ref{convexZ} shows the strict convexity of $Z_t(\mu)$ for any $t>0$. The case $t=0$ follows by direct differentiation: we get after simplification
\begin{eqnarray*}
Z''_t(\mu) & = & \frac{\left(\sum_{i\in [m]} u_i^2\right)\cdot\left(\sum_{i\in [m]} (q_{i} - \mu  u_i)^2\right)-\left(\sum_{i\in [m]} (q_{i} - \mu  u_i) u_i\right)^2}{\left(\sum_{i\in [m]} (q_{i} - \mu  u_i)^2\right)^{\frac{3}{2}}}.
\end{eqnarray*}
Cauchy-Schwartz inequality allows to conclude that $Z''_t(\mu)\geq 0$ and is in fact $>0$ \textit{unless} $\tilde{\ve{q}}$ is collinear to $\ve{u}$. This completes the proof of Theorem \ref{thENTBOOST}.

\subsection{Proof of Theorem  \ref{thBOOST1}}\label{proof-thBOOST1}

The proof involves several arguments, organised into several subsections. Some are more general than what is strictly needed for the proof of the Theorem, on purpose.
\subsubsection{Clamped summations}
For any $\delta \geq 0$, we define clamped summations of the sequence of ordered elements $v_1, v_2, ..., v_J$: if $J>1$,
\begin{eqnarray}
\sideset{^{(\delta)}}{}\sum^J_{j=1} v_j \defeq \min\left\{v_J + \sideset{^{(\delta)}}{}\sum^{J-1}_{j=1} v_j , \delta\right\} &, & \sideset{_{(-\delta)}}{}\sum^J_{j=1} v_j \defeq  \max\left\{v_J +\sideset{_{(-\delta)}}{}\sum^{J-1}_{j=1} v_j , -\delta\right\},
\end{eqnarray}
and the base case ($J=1$) is obtained by replacing the inner sum by 0. We also define the doubly clamped summation:
\begin{eqnarray*}
\sideset{^{(\delta)}_{(-\delta)}}{}\sum^J_{j=1} v_j \defeq \max\left\{\min\left\{v_J + \sideset{^{(\delta)}_{(-\delta)}}{}\sum^{J-1}_{j=1} v_j , \delta\right\} , -\delta\right\} ,
\end{eqnarray*}
with the same convention for the base case. We prove a series of simple but useful properties of the clamped summation.
\begin{lemma}\label{propcsum}
  The following properties hold true for clamped summation:
  \begin{enumerate}
  \item (doubly) clamped summations are non commutative;
  \item (doubly) clamped summations are ordinary summation in the limit: for any $J\geq 1$ and any sequence $v_1, v_2, ..., v_J$,
    \begin{eqnarray*}
\lim_{\delta \rightarrow +\infty} \sideset{^{(\delta)}}{}\sum^J_{j=1} v_j = \lim_{\delta \rightarrow +\infty} \sideset{_{(-\delta)}}{}\sum^J_{j=1} v_j = \lim_{\delta \rightarrow +\infty} \sideset{^{(\delta)}_{(-\delta)}}{}\sum^J_{j=1} v_j = \sum^J_{j=1} v_j 
      \end{eqnarray*}
    \item clamped summations sandwich ordinary summation and the doubly clamped summation:  for any $\delta \geq 0$, any $J\geq 1$ and any sequence $v_1, v_2, ..., v_J$,
\begin{eqnarray*}
  \sideset{^{(\delta)}}{}\sum^J_{j=1} v_j \leq \sum^J_{j=1} v_j \leq \sideset{_{(-\delta)}}{}\sum^J_{j=1} v_j & ; & \sideset{^{(\delta)}}{}\sum^J_{j=1} v_j \leq \sideset{^{(\delta)}_{(-\delta)}}{}\sum^J_{j=1} v_j \leq \sideset{_{(-\delta)}}{}\sum^J_{j=1} v_j
  \end{eqnarray*}
    \end{enumerate}
  \end{lemma}
  \begin{proof}
    Non commutativity follows from simple counterexamples: for example, for $v \defeq -1$ and $w \defeq 2$, if we fix $v_1 \defeq v, v_2 \defeq w$, then $\sideset{^{(0)}}{}\sum^2_{j=1} v_j = 1$ while $\sideset{^{(0)}}{}\sum^2_{j=1} v_{3-j} = -1$. Property [2.] is trivial. The set of leftmost inequalities of property [3.] can be shown by induction, noting the base case is trivial and otherwise, using the induction hypothesis in the leftmost inequality,
    \begin{eqnarray*}
      \sideset{^{(\delta)}}{}\sum^{J+1}_{j=1} v_j \defeq  \min\left\{v_{J+1} + \sideset{^{(\delta)}}{}\sum^{J}_{j=1} v_j  , \delta\right\} \leq \min\left\{v_{J+1} + \sum^{J}_{j=1} v_j  , \delta\right\} \leq  v_{J+1} + \sum^{J}_{j=1} v_j = \sum^{J+1}_{j=1} v_j,
    \end{eqnarray*}
    and similarly
    \begin{eqnarray*}
      \sideset{_{(-\delta)}}{}\sum^{J+1}_{j=1} v_j \defeq \max\left\{v_{J+1} + \sideset{_{(-\delta)}}{}\sum^{J}_{j=1} v_j  , -\delta\right\} \geq \max\left\{v_{J+1} + \sum^{J}_{j=1} v_j  , -\delta\right\} \geq v_{J+1} + \sum^{J}_{j=1} v_j = \sum^{J+1}_{j=1} v_j.
    \end{eqnarray*}
    A similar argument holds for the set of rightmost inequalities: for example, the induction's general case holds
    \begin{eqnarray*}
      \sideset{^{(\delta)}}{}\sum^{J+1}_{j=1} v_j & \defeq & \min\left\{ v_{J+1} + \sideset{^{(\delta)}}{}\sum^{J}_{j=1} v_j  , \delta\right\} \\
      & \leq & \min\left\{ v_{J+1} + \sideset{^{(\delta)}_{(-\delta)}}{}\sum^{J}_{j=1} v_j  , \delta\right\}  \leq   \max\left\{\min\left\{v_{J+1} +\sideset{^{(\delta)}_{(-\delta)}}{}\sum^{J}_{j=1} v_j  , \delta\right\},  -\delta\right\} = \sideset{^{(\delta)}_{(-\delta)}}{}\sum^J_{j=1} v_j.
      \end{eqnarray*}
for the leftmost inequality. This ends the proof of Lemma \ref{propcsum}.
\end{proof}
\subsubsection{Unravelling weights} 
\begin{lemma}\label{lemIUNRAVEL}
  Define
\begin{eqnarray}
v_{j} & \defeq & m^{1-t^*}\cdot \left(\prod_{k=1}^{j-1} Z_{tk}\right)^{1-t} \cdot \mu_j \quad\left(\mathrm{convention:} \prod_{k=1}^{0} u_k \defeq 1 \right).\label{defVJ}
\end{eqnarray}
Then $\forall J\geq 1$, weights unravel as:
\begin{eqnarray*}
  q_{(J+1)i} & = & \left\{
                           \begin{array}{rcl}
                             \frac{1}{m^{t^*} \prod_{j=1}^J Z_{tj}} \cdot \exp_t\left(-\sideset{^{(\nicefrac{1}{1-t})}}{}\sum^{J}_{j=1} v_j u_{ji}\right) & \mbox{ if } & t < 1\\
                             \frac{1}{m^{t^*} \prod_{j=1}^J Z_{tj}} \cdot \exp_t\left(-\sideset{_{(-\nicefrac{1}{1-t})}}{}\sum^{J}_{j=1} v_j u_{ji}\right) & \mbox{ if } & t > 1
                           \end{array} \right. .
  \end{eqnarray*}
\end{lemma}
\begin{proof}
 We start for the case $t<1$. We proceed by induction, noting first that the normalization constraint for the initial weights imposes $q_{1i} = 1/m^{1/(2-t)} = 1/m^{t^*}$ and so (using $(1-t)t^* = 1-t^*$)
  \begin{eqnarray*}
    q_{2i} & = & \frac{\exp_t(\log_t q_{1i} - \mu_1  u_{1i})}{Z_1}\\
                   & = & \frac{1}{Z_1}\cdot \left[1+(1-t) \cdot \left(\frac{1}{1-t}\cdot\left(\frac{1}{m^{\frac{1-t}{2-t}}}-1\right) -  \mu_1  u_{1i}\right) \right]_+^{\frac{1}{1-t}}\\
    & = & \frac{1}{Z_1}\cdot \left[\frac{1}{m^{1-t^*}} - (1-t)\mu_1  u_{1i}\right]_+^{\frac{1}{1-t}}\\
                   & = & \frac{1}{m^{t^*}Z_1}\cdot \left[1- (1-t)m^{1-t^*} \mu_1  u_{1i}\right]_+^{\frac{1}{1-t}}\\
    & = & \frac{1}{m^{t^*}Z_1}\cdot \exp_t\left(-\sideset{^{(\nicefrac{1}{1-t})}}{}\sum^{1}_{j=1} v_j u_{ji}\right),
    \end{eqnarray*}
    completing the base case. Using the induction hypothesis, we unravel at iteration $J+1$:
    {\footnotesize
\begin{eqnarray*}
  \lefteqn{q_{(J+1)i}}\nonumber\\
  & = & \frac{\exp_t(\log_t q_{Ji} - \mu_J  u_{Ji})}{Z_J} \nonumber \\
                     & = & \frac{\exp_t\left(\log_t\left(\frac{1}{m^{t^*} \prod_{j=1}^{J-1} Z_{tj}} \cdot \exp_t\left(-\sideset{^{(\nicefrac{1}{1-t})}}{}\sum^{J-1}_{j=1} v_j u_{ji}\right)\right) - \mu_J  u_{Ji}\right)}{Z_J} \nonumber \\
                     & = & \frac{\exp_t\left(\log_t\exp_t\left(-\sideset{^{(\nicefrac{1}{1-t})}}{}\sum^{J-1}_{j=1} v_j u_{ji}\right) \ominus_t \log_t \left(m^{t^*}\prod_{j=1}^{J-1} Z_{tj}\right) - \mu_J  u_{Ji}\right)}{Z_J}\nonumber \\
  & = & \frac{1}{Z_J} \cdot \exp_t\left(\frac{\max\left\{-\frac{1}{1-t}, -\sideset{^{(\nicefrac{1}{1-t})}}{}\sum^{J-1}_{j=1} v_j u_{ji}\right\} - \log_t \left(m^{t^*}\prod_{j=1}^{J-1} Z_{tj}\right)}{1+(1-t) \log_t \left(m^{t^*}\prod_{j=1}^{J-1} Z_{tj}\right)}- \mu_J  u_{Ji}\right)\nonumber \\
  & = & \frac{1}{Z_J} \cdot \left[1 + \frac{(1-t)\cdot \max\left\{-\frac{1}{1-t}, -\sideset{^{(\nicefrac{1}{1-t})}}{}\sum^{J-1}_{j=1} v_j u_{ji}\right\} - (1-t) \log_t \left(m^{t^*}\prod_{j=1}^{J-1} Z_{tj}\right)}{1+(1-t) \log_t \left(m^{t^*}\prod_{j=1}^{J-1} Z_{tj}\right)}- (1-t) \mu_J  u_{Ji}\right]^{\frac{1}{1-t}}_+\nonumber \\
  & = & \frac{1}{Z_J} \cdot \left[1 + \frac{(1-t)\cdot \max\left\{-\frac{1}{1-t}, -\sideset{^{(\nicefrac{1}{1-t})}}{}\sum^{J-1}_{j=1} v_j u_{ji}\right\} - \left(\left(m^{t^*}\prod_{j=1}^{J-1} Z_{tj}\right)^{1-t} - 1\right)}{\left(m^{t^*}\prod_{j=1}^{J-1} Z_{tj}\right)^{1-t}}- (1-t) \mu_J  u_{Ji}\right]^{\frac{1}{1-t}}_+,
\end{eqnarray*}
}
which simplifies into (using $(1-t)t^* = 1-t^*$)
\begin{eqnarray}
  \lefteqn{q_{(J+1)i}}\nonumber\\
  & = & \frac{1}{m^{t^*}\prod_{j=1}^{J} Z_{tj}} \cdot \left[1 + (1-t)\cdot \left(\max\left\{-\frac{1}{1-t}, -\sideset{^{(\nicefrac{1}{1-t})}}{}\sum^{J-1}_{j=1} v_j u_{ji}\right\}  - v_J u_{Ji}\right)\right]^{\frac{1}{1-t}}_+\label{eqFORSMALL}\\
  & = &  \frac{1}{m^{t^*}\prod_{j=1}^{J} Z_{tj}} \cdot \exp_t\left(-S_J\right),\nonumber
\end{eqnarray}
with
\begin{eqnarray*}
  S_J & \defeq & \min\left\{-\max\left\{-\frac{1}{1-t}, -\sideset{^{(\nicefrac{1}{1-t})}}{}\sum^{J-1}_{j=1} v_j u_{ji}\right\}  + v_J u_{Ji}, \frac{1}{1-t}\right\}\\
      & = & \min\left\{v_J u_{Ji} + \min\left\{\frac{1}{1-t}, \sideset{^{(\nicefrac{1}{1-t})}}{}\sum^{J-1}_{j=1} v_j u_{ji}\right\}  , \frac{1}{1-t}\right\}\\
      & = & \min\left\{v_J u_{Ji} + \sideset{^{(\nicefrac{1}{1-t})}}{}\sum^{J-1}_{j=1} v_j u_{ji} , \frac{1}{1-t}\right\}\\
  & \defeq & \sideset{^{(\nicefrac{1}{1-t})}}{}\sum^{J}_{j=1} v_j u_{ji}
\end{eqnarray*}
(we used twice the definition of clamped summation), which completes the proof of Lemma \ref{lemIUNRAVEL} for $t<1$.\\

\noindent We now treat the case $t>1$. The base induction is equivalent, while unravelling gives, instead of \eqref{eqFORSMALL}:
\begin{eqnarray}
  \lefteqn{q_{(J+1)i}}\nonumber\\
  & = & \frac{1}{m^{t^*}\prod_{j=1}^{J} Z_{tj}} \cdot \left[1 + (1-t)\cdot \left(\min\left\{-\frac{1}{1-t}, -\sideset{_{-(\nicefrac{1}{t-1})}}{}\sum^{J-1}_{j=1} v_j u_{ji}\right\}  - v_J u_{Ji}\right)\right]^{\frac{1}{1-t}}_+\nonumber\\
  & = &  \frac{1}{m^{t^*}\prod_{j=1}^{J} Z_{tj}} \cdot \exp_t\left(-S_J\right),\nonumber
\end{eqnarray}
and, this time,
\begin{eqnarray}
  S_J & \defeq & \max\left\{-\min\left\{-\frac{1}{1-t}, -\sideset{_{-(\nicefrac{1}{t-1})}}{}\sum^{J-1}_{j=1} v_j u_{ji}\right\}  + v_J u_{Ji}, -\frac{1}{t-1}\right\}\\
      & = & \max\left\{v_J u_{Ji} + \max\left\{-\frac{1}{t-1}, \sideset{_{-(\nicefrac{1}{t-1})}}{}\sum^{J-1}_{j=1} v_j u_{ji}\right\}, -\frac{1}{t-1}\right\}\\
      & = & \max\left\{v_J u_{Ji} + \sideset{_{-(\nicefrac{1}{t-1})}}{}\sum^{J-1}_{j=1} v_j u_{ji}, -\frac{1}{t-1}\right\}\\
  & \defeq & \sideset{_{(-\nicefrac{1}{t-1})}}{}\sum^{J}_{j=1} v_j u_{ji},
\end{eqnarray}
which completes the proof of Lemma \ref{lemIUNRAVEL}.
\end{proof}
\subsubsection{Introducing classifiers}
\paragraph{Ordinary linear separators} Suppose we have a classifier
\begin{eqnarray*}
H_J(\ve{x}) & \defeq & \sum_{j=1}^J \beta_j^{1-t} \mu_j \cdot h_j(\ve{x}), \quad \beta_j \defeq m^{t^*} \prod_{k=1}^{j-1} Z_{tk},
\end{eqnarray*}
where $\mu_j \in \mathbb{R}, \forall j\in[J]$. We remark that $\iver{z \neq r} \leq \exp_t^{2-t}(-zr)$ for any $t\leq 2$ and $z,r\in \mathbb{R}$, and $z\mapsto \exp_t^{2-t}(-z)$ is decreasing for any $t\leq 2$, so using [3.] in Lemma \ref{propcsum}, we get for our training sample $\mathcal{S} \defeq \{(\ve{x}_i, y_i), i \in [m]\}$ and any $t<1$ (from Lemma \ref{lemIUNRAVEL}),
{
\begin{eqnarray}
  \lefteqn{ \frac{1}{m} \cdot \sum_{i\in [m]} \iver{\mathrm{sign}(H_J(\ve{x}_i)) \neq y_i} }\nonumber\\
  & \leq & \sum_{i\in [m]} \frac{\exp_t^{2-t}\left(-\sum_{j=1}^{J} m^{1-t^*} \left(\prod_{k=1}^{j-1} Z_{tk}\right)^{1-t} \mu_j \cdot y_i h_j(\ve{x}_i) \right)}{m}\nonumber\\
                                                                  & \leq & \sum_{i\in [m]} \frac{\exp_t^{2-t}\left(-\sideset{^{(\nicefrac{1}{1-t})}}{}\sum_{j=1}^{J} m^{1-t^*} \left(\prod_{k=1}^{j-1} Z_{tk}\right)^{1-t} \mu_j \cdot y_i h_j(\ve{x}_i) \right)}{m}\label{bsup22}\\
  & & = \sum_{i\in [m]} \frac{\exp_t^{2-t}\left(-\sideset{^{(\nicefrac{1}{1-t})}}{}\sum_{j=1}^{J} v_j u_{ji}\right)}{m}\nonumber
\end{eqnarray}
}
where
\begin{eqnarray}
v_j \defeq m^{1-t^*} \left(\prod_{k=1}^{j-1} Z_{tk}\right)^{1-t} \mu_j & ; & u_{ji} \defeq y_i h_j(\ve{x}_i). \label{defVU}
\end{eqnarray}
Using Lemma \ref{lemIUNRAVEL} with those definitions, we get
\begin{eqnarray*}
  \frac{1}{m} \cdot \sum_{i\in [m]} \iver{\mathrm{sign}(H_J(\ve{x}_i)) \neq y_i} & \leq & \sum_{i\in [m]} \frac{\left(q_{(J+1)i} m^{t^*} \prod_{j=1}^J Z_{tj} \right)^{2-t}}{m}\\
  & & = \prod_{j=1}^J Z^{2-t}_{tj} \cdot \sum_{i\in [m]} q^{2-t}_{(J+1)i}\\
  & =& \prod_{j=1}^J Z^{2-t}_{tj} ,
\end{eqnarray*}
because $q_J \in \tilde{\Delta}_m$. We thus have proven the following Lemma.
\begin{lemma}\label{lemLS}
  For any $t < 1$ and any linear separator
  \begin{eqnarray*}
H_J(\ve{x}) & \defeq & \sum_{j=1}^J \beta_j^{1-t} \mu_j \cdot h_j(\ve{x}), \quad \left(\beta_j \defeq m^{t^*} \prod_{k=1}^{j-1} Z_{tk}, \mu_j \in \mathbb{R}, h_j \in \mathbb{R}^{\mathcal{X}}, \forall j \in [J]\right),
  \end{eqnarray*}
  where $Z_{tk}$ is the normalization coefficient of $\ve{q}$ in \eqref{defWopt} with $u_{ji} \defeq y_i h_j(\ve{x}_i)$,
  \begin{eqnarray}
  \frac{1}{m} \cdot \sum_{i\in [m]} \iver{\mathrm{sign}(H_J(\ve{x}_i)) \neq y_i} & \leq & \prod_{j=1}^J Z^{2-t}_{tj}.\label{bZOERRLS}
\end{eqnarray}
\end{lemma}
\paragraph{Clamped linear separators} Suppose we have a classifier ($t<1$)
\begin{eqnarray*}
H^{(\nicefrac{1}{1-t})}_J(\ve{x}) & \defeq & \sideset{^{(\nicefrac{1}{1-t})}_{(-\nicefrac{1}{1-t})}}{}\sum_{j=1}^J \beta_j^{1-t} \mu_j \cdot h_j(\ve{x}), \quad \beta_j \defeq m^{t^*} \prod_{k=1}^{j-1} Z_{tk}.
\end{eqnarray*}
We can now replace \eqref{bsup22} by
\begin{eqnarray}
  \lefteqn{ \frac{1}{m} \cdot \sum_{i\in [m]} \iver{\mathrm{sign}(H^{(\nicefrac{1}{1-t})}_J(\ve{x}_i)) \neq y_i}}\nonumber\\
  & \leq & \sum_{i\in [m]} \frac{\exp_t^{2-t}\left(-y_i \cdot \sideset{^{(\nicefrac{1}{1-t})}_{(-\nicefrac{1}{1-t})}}{}\sum_{j=1}^{J} m^{1-t^*} \left(\prod_{k=1}^{j-1} Z_{tk}\right)^{1-t} \mu_j \cdot  h_j(\ve{x}_i) \right)}{m}\nonumber\\
  & & = \sum_{i\in [m]} \frac{\exp_t^{2-t}\left(- \sideset{^{(\nicefrac{1}{1-t})}_{(-\nicefrac{1}{1-t})}}{}\sum_{j=1}^{J} m^{1-t^*} \left(\prod_{k=1}^{j-1} Z_{tk}\right)^{1-t} \mu_j \cdot  y_i h_j(\ve{x}_i) \right)}{m}\nonumber\\\nonumber\\
                                                                  & \leq & \sum_{i\in [m]} \frac{\exp_t^{2-t}\left(-\sideset{^{(\nicefrac{1}{1-t})}}{}\sum_{j=1}^{J} m^{1-t^*} \left(\prod_{k=1}^{j-1} Z_{tk}\right)^{1-t} \mu_j \cdot y_i h_j(\ve{x}_i) \right)}{m}\label{bsup23}\\
  & & = \sum_{i\in [m]} \frac{\exp_t^{2-t}\left(-\sideset{^{(\nicefrac{1}{1-t})}}{}\sum_{j=1}^{J} v_j u_{ji}\right)}{m}\nonumber.
\end{eqnarray}
The first identity has used the fact that $y_i \in \{-1,1\}$, so it can be folded in the doubly clamped summation without changing its value, and the second inequality used [3.] in Lemma \ref{propcsum}. This directly leads us to the following Lemma.
\begin{lemma}\label{lemClampedLS}
  For any $t < 1$ and any clamped linear separator
  \begin{eqnarray*}
H^{(\nicefrac{1}{1-t})}_J(\ve{x}) & \defeq & \sideset{^{(\nicefrac{1}{1-t})}_{(-\nicefrac{1}{1-t})}}{}\sum_{j=1}^J \beta_j^{1-t} \mu_j \cdot h_j(\ve{x}), \quad \left(\beta_j \defeq m^{t^*} \prod_{k=1}^{j-1} Z_{tk}, \mu_j \in \mathbb{R}, h_j \in \mathbb{R}^{\mathcal{X}}, \forall j \in [J]\right),
  \end{eqnarray*}
  where $Z_{tk}$ is the normalization coefficient of $\ve{q}$ in \eqref{defWopt} with $u_{ji} \defeq y_i h_j(\ve{x}_i)$,
  \begin{eqnarray}
  \frac{1}{m} \cdot \sum_{i\in [m]} \iver{\mathrm{sign}(H^{(\nicefrac{1}{1-t})}_J(\ve{x}_i)) \neq y_i} & \leq & \prod_{j=1}^J Z^{2-t}_{tj}.\label{bZOERRCLS}
\end{eqnarray}
\end{lemma}
\subsubsection{Geometric convergence of the empirical risk}
To get the right-hand side of \eqref{bZOERRLS} and \eqref{bZOERRCLS} as small as possible, we can independently compute each $\mu_j$ so as to minimize
\begin{eqnarray}
Z^{2-t}_{tj} (\mu) & \defeq & \sum_{i\in [m]} \exp^{2-t}_t\left(\log_t q_{ji} - \mu  u_{ji}\right). \label{defZMUj}
\end{eqnarray}
We proceed in two steps, first computing a convenient upperbound for \eqref{defZMUj}, and then finding the $\mu$ that minimizes this upperbound.

\noindent\textbf{Step 1}: We distinguish two cases depending on weight $q_{ji}$. Let $[m]_j^+ \defeq \{i : q_{ji} > 0\}$ and $[m]_j^\dagger \defeq \{i : q_{ji} = 0\}$:
\begin{enumerate}
\item [\textbf{Case 1}] $i \in [m]_j^+$. Let $r_{ji} \defeq u_{ji} / q_{ji}^{1-t}$ and suppose $R_j > 0$ is a real that satisfies
  \begin{eqnarray}
|r_{ji}| & \leq & R_j, \forall i\in [m]_j^+ \label{condRJ}.
    \end{eqnarray}
For any convex function $f$ defined on $[-1,1]$, we have $f(z) \leq ((1+z)/2)\cdot f(1) + ((1-z)/2) \cdot f(-1), \forall z \in [-1,1]$ (the straight line is the chord crossing $f$ at $z = -1, 1$). Because $z \mapsto [1-z]^{\frac{2-t}{1-t}}_+$ is convex for $t\leq 2$, for any $i \in [m]_j^+$
  \begin{eqnarray*}
    \lefteqn{\exp^{2-t}_t\left(\log_t q_{ji} - \mu  u_{ji}\right)}\\
    & = & \left[q_{ji}^{1-t} - (1-t) \mu u_{ji}\right]^{\frac{2-t}{1-t}}_+\\
                                                              & = & q_{ji}^{2-t} \cdot \left[1 - (1-t) \mu R_j \cdot \frac{r_{ji}}{R_j}\right]^{\frac{2-t}{1-t}}_+\\
                                                              & \leq & q_{ji}^{2-t} \cdot \frac{R_j +r_{ji}}{2R_j} \left[1 - (1-t) \mu R_j\right]^{\frac{2-t}{1-t}}_+ + q_{ji}^{2-t} \cdot \frac{R_j-r_{ji}}{2R_j} \left[1 + (1-t)\mu R_j\right]^{\frac{2-t}{1-t}}_+\\
    & & =  \frac{q_{ji}^{2-t}R_j + q_{ji} u_{ji}}{2R_j} \left[1 - (1-t) \mu R_j\right]^{\frac{2-t}{1-t}}_+ + \frac{q_{ji}^{2-t}R_j - q_{ji} u_{ji}}{2R_j} \left[1 + (1-t)\mu R_j \right]^{\frac{2-t}{1-t}}_+.
  \end{eqnarray*}
\item [\textbf{Case 2}] $i \in [m]_j^\dagger$. Let ${q^\dagger_j} > 0$ be a real that satisfies
  \begin{eqnarray}
\frac{|u_{ji}|}{{q^\dagger_j}^{1-t}} < R_j, \forall i \in [m]_j^\dagger. \label{condEpsilonJ}
    \end{eqnarray}
Using the same technique as in case 1, we find for any $i \in [m]_j^\dagger$
  \begin{eqnarray*}
    \lefteqn{\exp^{2-t}_t\left(\log_t q_{ji} - \mu  u_{ji}\right) }\\
    & = & \exp^{2-t}_t\left(-\frac{1}{1-t} - \mu  u_{ji}\right)\\
    & = & \left[ - (1-t) \mu u_{ji}\right]^{\frac{2-t}{1-t}}_+\\
    & \leq & \left[{q^\dagger_j}^{1-t} - (1-t) \mu u_{ji}\right]^{\frac{2-t}{1-t}}_+\\
                                                              & \leq & \frac{{q^\dagger_j}^{2-t}R_j + {q^\dagger_j} u_{ji}}{2R_j} \left[1 - (1-t) \mu R_j\right]^{\frac{2-t}{1-t}}_+ + \frac{{q^\dagger_j}^{2-t}R_j - {q^\dagger_j} u_{ji}}{2R_j} \left[1 + (1-t)\mu R_j \right]^{\frac{2-t}{1-t}}_+.
  \end{eqnarray*}
  \end{enumerate}
  Folding both cases into one and letting
  \begin{eqnarray}
    {q'}_{ji} & \defeq & \left\{
                                      \begin{array}{rcl}
                                        q_{ji} & \mbox{ if } & i \in [m]_j^+\\
                                        {q^\dagger_j} & \mbox{ if } & i \in [m]_j^\dagger
                                        \end{array}
                                      \right.,
  \end{eqnarray}
  we get after summation, using $m_j^\dagger \defeq \mathrm{Card}([m]_j^\dagger)$ and
  \begin{eqnarray}
    \rho_{j} & \defeq & \frac{1}{(1+m_j^\dagger {q^\dagger_j}^{2-t})R_j} \cdot \sum_{i\in [m]} {q'}_{ji} u_{ji} \quad (\in [-1,1]), \label{defRHOtj}
  \end{eqnarray}
  that
  \begin{eqnarray}
    \lefteqn{Z^{2-t}_{tj} (\mu)}\nonumber\\
    & \leq & \frac{(1+m_j^\dagger {q^\dagger_j}^{2-t})R_j}{2R_j} \cdot \left( (1+\rho_{j}) \left[1 - (1-t) \mu R_j\right]^{\frac{2-t}{1-t}}_+ + (1-\rho_{j}) \left[1 + (1-t) \mu R_j\right]^{\frac{2-t}{1-t}}_+ \right)\nonumber\\
    & & = \frac{1+m_j^\dagger {q^\dagger_j}^{2-t}}{2} \cdot \left( (1+\rho_{j}) \cdot \exp_{t}^{2-t} \left(-\mu R_j\right) + (1-\rho_{j}) \cdot \exp_{t}^{2-t} \left(\mu R_j\right) \right).\label{bZAda}
  \end{eqnarray}
\noindent\textbf{Step 2}: we have our upperbound for \eqref{defZMUj}. We now compute the minimizer $\mu^*$ of \eqref{bZAda}. If this minimizer satisfies
    \begin{eqnarray}
|\mu^*| & < & \frac{1}{R_j|1-t|}, \label{condMUSTAR}
    \end{eqnarray}
    then it can be found by ordinary differentiation, as the solution to
    \begin{eqnarray*}
 (1-\rho_{j}) \cdot \exp_{t} \left(\mu^* R_j\right) -  (1+\rho_{j}) \cdot \exp_{t} \left(-\mu^* R_j\right) & = & 0,
      \end{eqnarray*}
      which is equivalently
      \begin{eqnarray*}
        \frac{\exp_{t} \left(-\mu^* R_j\right)}{\exp_{t} \left(\mu^* R_j\right)} & = & \exp_t\left(-\mu^* R_j \ominus_t \mu^* R_j\right)\\
        & = & \frac{1-\rho_{j}}{1+\rho_{j}},
      \end{eqnarray*}
      where we recall $a \ominus_t b \defeq (a-b)/(1+(1-t)b)$. Solving it yields
      \begin{eqnarray*}
        \mu^* & = & \frac{1}{R_j} \cdot -\frac{1}{1-t} \cdot \left(\frac{(1-\rho_{j})^{1-t} - (1+\rho_{j})^{1-t}}{(1-\rho_{j})^{1-t} + (1+\rho_{j})^{1-t}}\right)\\
            & = & \frac{1}{R_j} \cdot -\frac{1}{1-t} \cdot \left(\frac{2(1-\rho_{j})^{1-t}}{(1-\rho_{j})^{1-t} + (1+\rho_{j})^{1-t}}-1\right)\\
        & = & -\frac{1}{R_j} \cdot \log_t \left(\frac{1-\rho_{j}}{M_{1-t}(1-\rho_{j},1+\rho_{j})}\right),
      \end{eqnarray*}
      where $M_q(a,b) \defeq ((a^q+b^q)/2)^{1/q}$ is the power mean with exponent $q$. We now check \eqref{condMUSTAR}.
      \begin{lemma}\label{lemBMUj}
        For any $t \in \mathbb{R}$, let
        \begin{eqnarray}
\mu_j & \defeq & -\frac{1}{R_j} \cdot \log_t \left(\frac{1-\rho_{j}}{M_{1-t}(1-\rho_{j},1+\rho_{j})}\right).\label{defMUj}
        \end{eqnarray}
        Then $|\mu_j| \leq 1/(R_j|1-t|)$.
        \end{lemma}
\begin{proof}
Equivalently, we must show 
      \begin{eqnarray*}
\left| \log_t \left(\frac{1-z}{M_{1-t}(1-z,1+z)}\right) \right| & \leq & \frac{1}{|1-t|} , \forall z \in [-1,1],
      \end{eqnarray*}
      which is equivalent to showing
      \begin{eqnarray*}
\left|\frac{2(1-z)^{1-t}}{(1-z)^{1-t}+(1+z)^{1-t}}-1\right| \left(= \left|\frac{1-\left(\frac{1+z}{1-z}\right)^{1-t}}{1+\left(\frac{1+z}{1-z}\right)^{1-t}}\right|\right)& \leq & 1 , \forall z \in [-1,1].
      \end{eqnarray*}
      Define function $f(z,t) \defeq (1-z^{1-t})/(1+z^{1-t})$ over $\mathbb{R}_{\geq 0}\times \mathbb{R}$: it is easy to check that for $t\leq 1, f(z,t) \in [-1,1]$, and the symmetry $f(z,t) = -f(z,2-t)$ also allows to conclude that for $t\geq 1, f(z,t) \in [-1,1]$. This ends the proof of Lemma \ref{lemBMUj}.
    \end{proof}
    \noindent For the expression of $\mu_j$ in \eqref{defMUj}, we get from \eqref{bZAda} the upperbound on $Z^{2-t}_{tj} (\mu_j)$:
    \begin{eqnarray*}
      Z^{2-t}_{tj} (\mu_j) & \leq & \frac{1+m_j^\dagger {q^\dagger_j}^{2-t}}{2} \cdot \left( (1+\rho_{j}) \cdot \exp_{t}^{2-t} \left(-\mu_j R_j\right) + (1-\rho_{j}) \cdot \exp_{t}^{2-t} \left(\mu_j R_j\right) \right)\\
      & & = \frac{1+m_j^\dagger {q^\dagger_j}^{2-t}}{2} \cdot \left( \frac{(1+\rho_{j})(1-\rho_{j})^{2-t}}{M^{2-t}_{1-t}(1-\rho_{j},1+\rho_{j})} + \frac{(1-\rho_{j})(1+\rho_{j})^{2-t}}{M^{2-t}_{1-t}(1-\rho_{j},1+\rho_{j})} \right)\\
                                & = & \left(1+m_j^\dagger {q^\dagger_j}^{2-t}\right) \cdot \frac{(1-\rho^2_{j})M^{1-t}_{1-t}(1-\rho_{j},1+\rho_{j})}{M^{2-t}_{1-t}(1-\rho_{j},1+\rho_{j})} \\
      & = & \left(1+m_j^\dagger {q^\dagger_j}^{2-t}\right) \cdot \frac{1-\rho^2_{j}}{M_{1-t}(1-\rho_{j},1+\rho_{j})}.
      \end{eqnarray*}
      We conclude that for both sets of classifiers defined in Lemmata \ref{lemLS} and \ref{lemClampedLS}, with the choice of $\mu_j$ in \eqref{defMUj}, we get
      \begin{eqnarray*}
  \frac{1}{m} \cdot \sum_{i\in [m]} \iver{\mathrm{sign}(H(\ve{x}_i)) \neq y_i} & \leq & \prod_{j=1}^J \left(1+m_j^\dagger {q^\dagger_j}^{2-t}\right) \cdot \frac{1-\rho^2_{j}}{M_{1-t}(1-\rho_{j},1+\rho_{j})}, \forall H \in \{H_J, H^{(\nicefrac{1}{1-t})}_J\}.
      \end{eqnarray*}
      To complete the proof of Theorem \ref{thBOOST1}, we just need to elicit the best $R_j$ \eqref{condRJ} and ${q^\dagger_j}$ \eqref{condEpsilonJ}; looking at their constraints suggests
      \begin{eqnarray*}
        R_j & \defeq & \max_{i \not\in [m]_j^\dagger} \frac{|y_i h_j(\ve{x}_i)|}{q_{ji}^{1-t}},\\
        {q^\dagger_j} & \defeq & \frac{\max_{i \in [m]_j^\dagger} |y_i h_j(\ve{x}_i)|^{1/(1-t)}}{R^{1/(1-t)}_j}.
      \end{eqnarray*}
      This completes the proof of Theorem \ref{thBOOST1}. We complete the proof by a few additional useful results in the context of Algorithm \tada.
      \begin{lemma}\label{lemADDR}
The following holds true: (i) $\rho_{j} \in [-1,1]$; (ii) if, among indexes not in $[m]_j^\dagger$, there exists at least one index with $u_{ji} > 0$ and one index with $u_{ji} < 0$, then for any $\mu \neq 0$, $Z^{2-t}_{tj} (\mu) > 0$ in \eqref{defZMUj} (in words, the new weigh vector $\ve{q}_{j+1}$ cannot be the null vector before normalization).
\end{lemma}
\begin{proof}
To show (i) for $\rho_{j} \leq 1$, we write (using $u_{ji} \defeq y_i h_j(\ve{x}_i), \forall i \in [m]$ for short),
\begin{eqnarray*}
  (1+m_j^\dagger {q^\dagger_j}^{2-t})R_j \cdot \rho_{j} & = & \sum_{i\in [m]} {q'}_{ji} u_{ji} \\
           & \leq & \sum_{i\in [m]} {q'}_{ji} |u_{ji}|\\
                                                        & & = \sum_{i\in [m]_j^+} {q}^{2-t}_{ji} \cdot \frac{|u_{ji}|}{{q}^{1-t}_{ji}} + {q^\dagger_j}^{2-t} \cdot\frac{\sum_{i\in [m]_j^\dagger} |u_{ji}|}{{q^\dagger_j}^{1-t}}\\
                                                        & \leq & R_j \cdot \underbrace{\sum_{i\in [m]_j^+} {q}^{2-t}_{ji}}_{=1} + {q^\dagger_j}^{2-t} \cdot \frac{R_j \sum_{i\in [m]_j^\dagger} |u_{ji}|}{\max_{i \in [m]_j^\dagger} |u_{ji}|} \\
  & \leq & R_j + {q^\dagger_j}^{2-t} m_j^\dagger R_j  = (1+m_j^\dagger {q^\dagger_j}^{2-t})R_j,
\end{eqnarray*}
showing $\rho_{j} \leq 1$. Showing $\rho_{j} \geq -1$ proceeds in the same way. Property (ii) is trivial.
\end{proof}

\begin{lemma}
  \begin{eqnarray*}
K_t(z) & \leq & \exp\left(-\left(1-\frac{t}{2}\right)\cdot z^2\right)
    \end{eqnarray*}
\end{lemma}
\begin{proof}
  We remark that for $t\in [0,1), z\geq 0$, $K'_t(z)$ is concave and $K''_t(0) = -(2-t)$, so $K'_t(z) \leq -(2-t) z, \forall z\geq 0$, from which it follows by integration
  \begin{eqnarray*}
K_t(z) & \leq &  1 - \left(1-\frac{t}{2}\right)\cdot z^2
  \end{eqnarray*}
  and since $1-z\leq \exp(-z)$, we get the statement of the Lemma.
  \end{proof}

  \subsection{Proof of Theorem  \ref{th-CPE}}\label{proof-th-CPE}

  The proof proceeds in three parts. Part \textbf{(A)} makes a brief recall on encoding linear classifiers with decision trees. Part \textbf{(B)} solves \eqref{eqMU} in \mainfile, \textit{i.e.} finds boosting's leveraging coefficients as solution of:
  \begin{eqnarray}
    \ve{q}(\mu)^\top \ve{u} & = & 0.\label{eqMUSI}
  \end{eqnarray}
  we then simplify the loss obtained and elicit the conditional Bayes risk of the tempered loss, \textit{i.e.} \eqref{pcbr} in \mainfile. Part \textbf{(C)} elicits the partial losses and shows properness and related properties.

\begin{figure}[t]
  \begin{center}
    \includegraphics[trim=0bp 0bp 0bp 0bp,clip,width=\columnwidth]{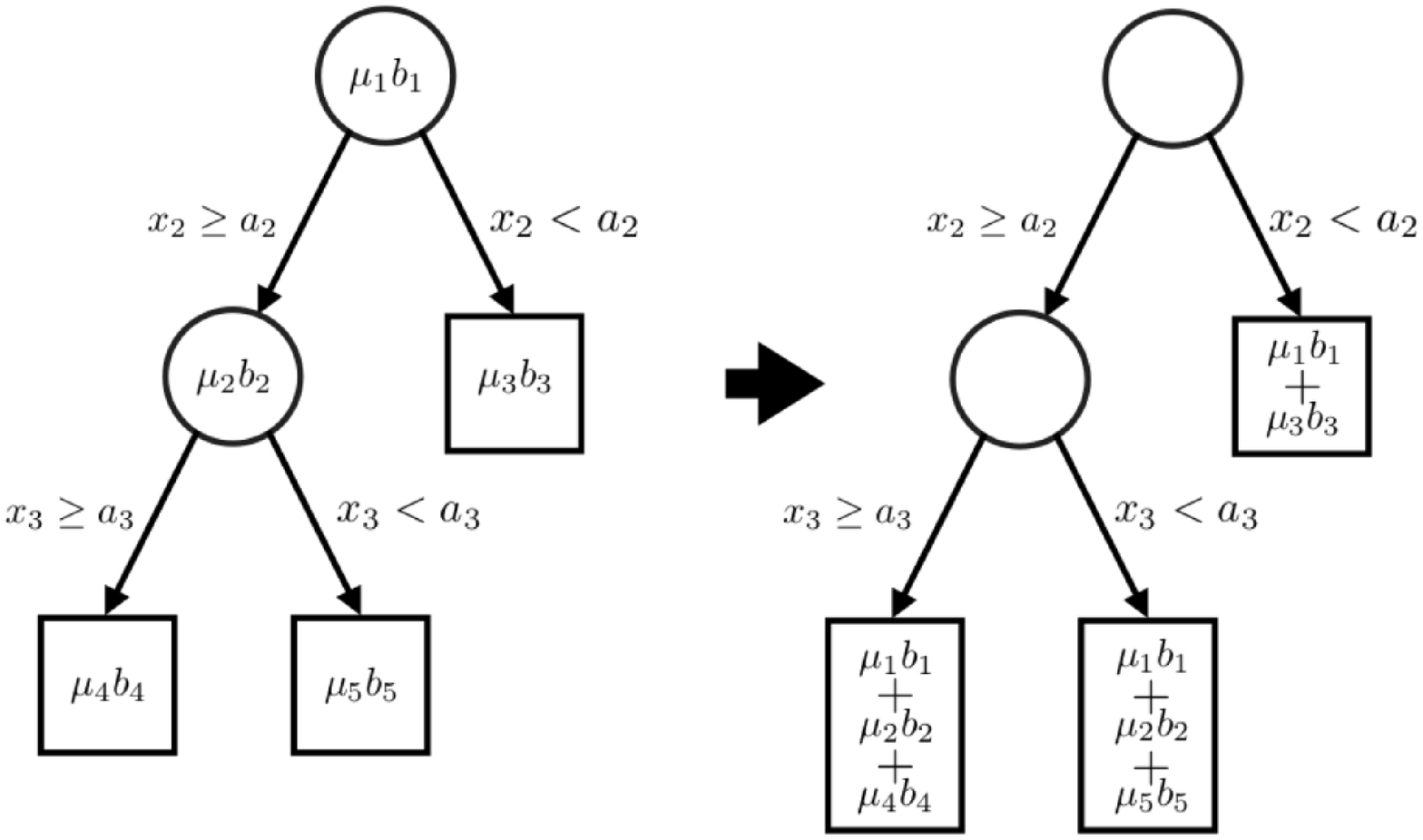}
    \end{center}
\caption{The weak learner provides weak hypotheses of the form $\iver{x_k \geq a_j} \cdot b_j$. From the boosting standpoint, this weak hypothesis is "as good as" the weak hypothesis $\overline{h}_j(\ve{x}) \defeq \iver{x_k < a_j} \cdot -b_j$. The predicates of both are used to craft a split, \textit{e.g.} for the root (in our depiction, $b_3 = -b_2$) and then solving \eqref{eqMUSI} provides the leveraging coefficients $\mu_.$. We then repeat this for as many splits as necessary. At the end, we can "percolate" nodes reals towards the leaves below and get an equivalent classifier that resembles a decision tree (right). See \cite{mnwRC} for further details.}
\label{f-DT-transformation}
\end{figure}

\paragraph{Part (A): encoding linear models with a tree architecture} We use the reduction trick of \cite{mnwRC} to design a decision tree (\cdt) boosting procedure, find out the (concave) loss equivalently minimized, just like in classical top-down \cdt~induction algorithms \cite{bfosCA}. The trick is simple: a \cdt~can be thought of as a set of constant linear classifiers. The prediction is the sum of predictions put at all nodes. Boosting fits those predictions at the nodes and percolating those to leaves gets a standard \cdt~with real predictions at the leaves. Figure \ref{f-DT-transformation} provides a detailed description of the procedure. 
  Let $\leaf$ denote a leaf node of the current tree $H$, with $H_\leaf \in \mathbb{R}$ the function it implements for leaf $\leaf$. If $\parent(\leaf)$ denotes its parent node (assuming wlog it is not the root node), we have
\begin{eqnarray}
H_\leaf & \defeq & H_{\parent(\leaf)} +  \mu_\leaf  h_\leaf, \label{defdtclass}
\end{eqnarray}
\paragraph{Part (B): eliciting the Bayes risk of the tempered loss} With our simple classifiers at hand, the tempered exponential loss $Z^{2-t}_{tj}$ in \eqref{defTEXPloss} (\mainfile) can be simplified to loss
  \begin{eqnarray}
    L(H)  & \defeq & \sum_i \exp^{2-t}_t\left(\log_t q_{1i} - y_i H_{\leaf(\ve{x}_i)} \right)\nonumber\\
    & = & \sum_{\leaf \in \leafset(H)} m^+_\leaf \exp^{2-t}_t\left(\log_t q_{1i} - H_{\leaf} \right) + m^-_\leaf \exp^{2-t}_t\left(\log_t q_{1i} + H_{\leaf} \right) \label{decompLOSS},
\end{eqnarray}
where $\leaf(\ve{x})$ is the leaf reached by observation $\ve{x}$ and $\leaf(H)$ its set of leaf nodes of $H$, and $H_{\leaf}$ sums all relevant values in \eqref{defdtclass}. Also, $m^+_\leaf, m^-_\leaf$ denote the cardinal of positive and negative examples at $\leaf$ and $p_\leaf \defeq m^+_\leaf / (m^+_\leaf + m^-_\leaf)$ the local proportion of positive examples at $\leaf$, and finally $r_\leaf \defeq (m^+_\leaf + m^-_\leaf) / m$ the total proportion of examples reaching $\leaf$. 
  \begin{theorem}\label{thCPELOSS}
    If we compute $\mu_\leaf$ the solution of \eqref{eqMUSI}, we end up with the prediction $H_\leaf$:
    \begin{eqnarray}
  H_\leaf & = & \frac{q^{1-t}_{1i}}{1-t} \cdot \frac{\left(\frac{m^+_\leaf}{m^-_\leaf}\right)^{1-t} - 1}{\left(\frac{m^+_\leaf}{m^-_\leaf}\right)^{1-t} + 1}\\
  & = & \frac{q^{1-t}_{1i}}{1-t} \cdot \frac{p_\leaf^{1-t} - (1-p_\leaf)^{1-t}}{p_\leaf^{1-t} + (1-p_\leaf)^{1-t}},\label{predHLeaf}
    \end{eqnarray}
    and the loss of the decision tree equals:
    \begin{eqnarray}
      L(H) & = & \sum_{\leaf \in \leafset(H)} r_\leaf \cdot \frac{2 p_\leaf (1-p_\leaf)}{M_{1-t}(p_\leaf, 1-p_\leaf)},\label{cbrTM}\\
      & = & \expect_{\leaf} [\bayestrisk(p_\leaf)] \label{cbrTL}.
\end{eqnarray}
    \end{theorem}
\begin{proof}
   To compute $\mu_\leaf$, \eqref{eqMU} is reduced to the examples reaching $\leaf$, that is, it simplifies to
  \begin{eqnarray}
m^+_\leaf \exp_t\left(\log_t q_{1i} -  H_{\parent(\leaf)} -  R_\leaf \mu_\leaf  h_\leaf\right) & = & m^-_\leaf \exp_t\left(\log_t q_{1i} +  H_{\parent(\leaf)} +  R_\leaf \mu_\leaf  h_\leaf\right)\,\label{eqMUTREE}
  \end{eqnarray}
  that we solve for $\mu_\leaf$.  Equivalently,
\begin{eqnarray*}
\frac{\exp_t\left(\log_t q_{1i} +  H_{\parent(\leaf)} +  R_\leaf \mu_\leaf  h_\leaf\right)}{\exp_t\left(\log_t q_{1i} -  H_{\parent(\leaf)} -  R_\leaf \mu_\leaf  h_\leaf\right)} & = & \frac{m^+_\leaf}{m^-_\leaf},
\end{eqnarray*}
or, using $\exp_t(u) / \exp_t(v) = \exp_t(u\ominus_t v)$,
\begin{eqnarray*}
\frac{2H_{\parent(\leaf)} +  2R_\leaf \mu_\leaf  h_\leaf}{1 + (1-t) (\log_t q_{1i} -  H_{\parent(\leaf)} -  R_\leaf \mu_\leaf  h_\leaf)} & = & \log_t\left(\frac{m^+_\leaf}{m^-_\leaf}\right),
\end{eqnarray*}
after reorganizing:
\begin{eqnarray*}
  R_\leaf \mu_\leaf  h_\leaf  & = & \frac{(1 + (1-t) (\log_t q_{1i} -  H_{\parent(\leaf)}))\cdot \log_t\left(\frac{m^+_\leaf}{m^-_\leaf}\right) - 2H_{\parent(\leaf)}}{2 + (1-t) \log_t\left(\frac{m^+_\leaf}{m^-_\leaf}\right)},
\end{eqnarray*}
which yields the prediction at $\leaf$:
\begin{eqnarray}
  H_\leaf & = & H_{\parent(\leaf)} + \frac{(1 + (1-t) (\log_t q_{1i} -  H_{\parent(\leaf)}))\cdot \log_t\left(\frac{m^+_\leaf}{m^-_\leaf}\right) - 2H_{\parent(\leaf)}}{2 + (1-t) \log_t\left(\frac{m^+_\leaf}{m^-_\leaf}\right)}\\
          & = & \frac{(1 + (1-t) \log_t q_{1i} )\cdot \log_t\left(\frac{m^+_\leaf}{m^-_\leaf}\right) }{2 + (1-t) \log_t\left(\frac{m^+_\leaf}{m^-_\leaf}\right)}\\
          & = & q^{1-t}_{1i} \cdot \frac{\log_t\left(\frac{m^+_\leaf}{m^-_\leaf}\right)}{2 + (1-t)\log_t\left(\frac{m^+_\leaf}{m^-_\leaf}\right)}\\
  & = & \frac{q^{1-t}_{1i}}{1-t} \cdot \frac{\left(\frac{m^+_\leaf}{m^-_\leaf}\right)^{1-t} - 1}{\left(\frac{m^+_\leaf}{m^-_\leaf}\right)^{1-t} + 1}\\
  & = & \frac{q^{1-t}_{1i}}{1-t} \cdot \frac{p_\leaf^{1-t} - (1-p_\leaf)^{1-t}}{p_\leaf^{1-t} + (1-p_\leaf)^{1-t}}.\label{predHLeaf}
\end{eqnarray}
We plug $H_\leaf$ back in the loss for all leaves and get, using $q_{1i} = 1/m^{1/(2-t)}$:
\begin{eqnarray}
  L(H)   & = & \sum_{\leaf \in \leafset(H)} \left\{\begin{array}{c}
                                                                        m^+_\leaf \exp_t^{2-t} \left(\log_t q_{1i} - \frac{q^{1-t}_{1i}}{1-t} \cdot \frac{p_\leaf^{1-t} - (1-p_\leaf)^{1-t}}{p_\leaf^{1-t} + (1-p_\leaf)^{1-t}}\right) \\
                                                                        +  m^-_\leaf \exp_t^{2-t} \left(\log_t q_{1i} + \frac{q^{1-t}_{1i}}{1-t} \cdot \frac{p_\leaf^{1-t} - (1-p_\leaf)^{1-t}}{p_\leaf^{1-t} + (1-p_\leaf)^{1-t}}\right)
                                                                        \end{array}\right.
\end{eqnarray}
We simplify. First,
\begin{eqnarray}
  \lefteqn{m^+_\leaf \exp_t^{2-t} \left(\log_t q_{1i} - \frac{q^{1-t}_{1i}}{1-t} \cdot \frac{p_\leaf^{1-t} - (1-p_\leaf)^{1-t}}{p_\leaf^{1-t} + (1-p_\leaf)^{1-t}}\right) }\nonumber\\
  & = & m^+_\leaf \left[q^{1-t}_{1i} \cdot \left(1 - \frac{p_\leaf^{1-t} - (1-p_\leaf)^{1-t}}{p_\leaf^{1-t} + (1-p_\leaf)^{1-t}}\right)\right]_+^{\frac{2-t}{1-t}}\nonumber\\
  & = & \frac{m^+_\leaf}{m}\cdot \left[\frac{2(1-p_\leaf)^{1-t}}{p_\leaf^{1-t} + (1-p_\leaf)^{1-t}}\right]_+^{\frac{2-t}{1-t}}\\
  & = & \frac{m^+_\leaf}{m}\cdot \left(\frac{1-p_\leaf}{M_{1-t}(p_\leaf, 1-p_\leaf)}\right)^{2-t},
\end{eqnarray}
and then
\begin{eqnarray}
  \lefteqn{m^-_\leaf \exp_t^{2-t} \left(\log_t q_{1i} + \frac{q^{1-t}_{1i}}{1-t} \cdot \frac{p_\leaf^{1-t} - (1-p_\leaf)^{1-t}}{p_\leaf^{1-t} + (1-p_\leaf)^{1-t}}\right) }\nonumber\\
  & = & m^-_\leaf \left[q^{1-t}_{1i} \cdot \left(1 + \frac{p_\leaf^{1-t} - (1-p_\leaf)^{1-t}}{p_\leaf^{1-t} + (1-p_\leaf)^{1-t}}\right)\right]_+^{\frac{2-t}{1-t}}\nonumber\\
  & = & \frac{m^-_\leaf}{m}\cdot \left[\frac{2 p_\leaf^{1-t}}{p_\leaf^{1-t} + (1-p_\leaf)^{1-t}}\right]_+^{\frac{2-t}{1-t}}\\
  & = & \frac{m^-_\leaf}{m}\cdot \left(\frac{p_\leaf}{M_{1-t}(p_\leaf, 1-p_\leaf)}\right)^{2-t},
\end{eqnarray}
and we can simplify the loss, 
\begin{eqnarray}
  L(H) & = & \sum_{\leaf \in \leafset(H)} r_\leaf p_\leaf\left(\frac{1-p_\leaf}{M_{1-t}(p_\leaf, 1-p_\leaf)}\right)^{2-t} + r_\leaf (1-p_\leaf) \left(\frac{p_\leaf}{M_{1-t}(p_\leaf, 1-p_\leaf)}\right)^{2-t}\label{suggPL}\\
  & = &  \sum_{\leaf \in \leafset(H)} r_\leaf \cdot \frac{p_\leaf (1-p_\leaf)^{2-t} + (1-p_\leaf) p_\leaf^{2-t}}{M^{2-t}_{1-t}(p_\leaf, 1-p_\leaf)}  \\
  & = &  \sum_{\leaf \in \leafset(H)} r_\leaf \cdot \frac{p_\leaf (1-p_\leaf) \cdot (p_\leaf^{1-t} + (1-p_\leaf)^{1-t})}{M^{2-t}_{1-t}(p_\leaf, 1-p_\leaf)}  \\
  & = &  \sum_{\leaf \in \leafset(H)} r_\leaf \cdot \frac{2 p_\leaf (1-p_\leaf) \cdot M^{1-t}_{1-t}(p_\leaf, 1-p_\leaf)}{M^{2-t}_{1-t}(p_\leaf, 1-p_\leaf)}  \\
  & = &  \sum_{\leaf \in \leafset(H)} r_\leaf \cdot \frac{2 p_\leaf (1-p_\leaf)}{M_{1-t}(p_\leaf, 1-p_\leaf)},
\end{eqnarray}
as claimed. This ends the proof of Theorem \ref{thCPELOSS}.
\end{proof}
\paragraph{Part (C): partial losses and their properties} The proof relies on the following Theorem. We recall that a loss is symmetric iff its partial losses satisfy $\partialloss{1}(u) = \partialloss{-1}(1-u), \forall u\in [0,1]$ \cite{nnOT} and differentiable iff its partial losses are differentiable.
\begin{theorem}
  Suppose $t<2$. A set of partial losses having the conditional Bayes risk $\bayestrisk$ in \eqref{cbrTL} are 
\begin{eqnarray}
\partialtloss{1}(u) \defeq \left(\frac{1-u}{M_{1-t}(u, 1-u)}\right)^{2-t} & , & \partialtloss{-1}(u) \defeq \partialtloss{1}(1-u).
\end{eqnarray}
The tempered loss is then symmetric and differentiable. It is strictly proper for any $t \in (-\infty, 2)$ and proper for $t = -\infty$.
\end{theorem}
\begin{proof}
  Symmetry and differentiability are straightforward. To check strict properness, we analyze the cases for $t\neq 1$ (otherwise, it is Matusita's loss, and thus strictly proper), we compute the solution $u$ to
  \begin{eqnarray}
\frac{\partial}{\partial u} \poirisk(u,v) & = & 0.\label{eqpu}
  \end{eqnarray}
  To this end, let $N(u) \defeq v (1-u)^{2-t} + (1-v) u^{2-t}$ and the $q$-sum
  \begin{eqnarray}
S_{q}(a,b) & \defeq & (a^{q} + b^{q})^{1/q} = 2^{1/q} \cdot M_q(a,b).
    \end{eqnarray}
We also let $D(u) \defeq S_{1-t}^{2-t}(u, 1-u)$. Noting $\poitrisk(u,v) = 2^{\frac{2-t}{1-t}} \cdot N(u) / D(u)$ and $D(u) \neq 0, \forall u\in [0,1]$, the set of solutions of \eqref{eqpu} are the set of solutions to $N'(u) D(u) = N(u) D'(u)$, which boils down, after simplification, to
  \begin{eqnarray*}
    \lefteqn{((1-v)u^{1-t} - v(1-u)^{1-t}) S_{1-t}^{2-t}(u, 1-u)}\\
    & = & (v(1-u)^{2-t} + (1-v)u^{2-t})(u^{-t} - (1-u)^{-t}) S_{1-t}(u, 1-u),
  \end{eqnarray*}
  developing and simplifying yields a first simplified expression $(1-2v)(u(1-u))^{1-t} = vu^{-1}(1-u)^{2-t} - (1-v)(1-u)^{-t}u^{2-t}$, which, after reorganising to isolate expressions depending on $v$, yields
  \begin{eqnarray}
(u(1-u))^{1-t} + (1-u)^{-t}u^{2-t} & = & v \cdot \left(u^{-t}(1-u)^{2-t} + (1-u)^{-t}u^{2-t} + 2(u(1-u))^{1-t}\right). \label{eqUV1}
  \end{eqnarray}
  Assuming $v \in (0,1)$, we multiply by $(u(1-u))^{1-t}$ (we shall check $u \in (0,1)$) and simplify, which yields $u(1-u) + u^2 = v((1-u)^2 + u^2 + 2u(1-u))$, and indeed yields
  \begin{eqnarray}
u & = & v,
  \end{eqnarray}
  and we check from \eqref{eqUV1} that if $v=0$ (resp. $v=1$), then necessarily $u=0$ (resp. $u=1$). To complete the proof, using the previous derivations, we can then simplify
  \begin{eqnarray}
\frac{\partial}{\partial u} \poirisk(u,v) & = & (2-t) \cdot 2^{\frac{2-t}{1-t}} \cdot \frac{u-v}{(u(1-u))^t \cdot S_{1-t}^{3-2t}(u, 1-u)},
  \end{eqnarray}
  which shows that if $2-t > 0$ but $t \neq -\infty$, $u=v$ is a strict minimum of the pointwise conditional risk, completing the proof for strict properness. Strict properness is sufficient to show by a simple computation that $\bayestrisk$ is \eqref{cbrTL}. For $t=-\infty$, we pass to the limit and use the fact that we can also write
\begin{eqnarray}
\partialtloss{1}(u) & = & \frac{1}{M_{1-t^*}\left(1,\left(\frac{u}{1-u}\right)^{\frac{1}{t^*}}\right)} \quad (\mbox{we recall }t^* \defeq 1/(2-t))
\end{eqnarray}
$t\rightarrow -\infty$ is equivalent to $t^* \rightarrow 0^+$. If $u<1/2$, $u/(1-u) < 1$ and so we see that
\begin{eqnarray*}
  \lim_{t^* \rightarrow 0^+} M_{1-t^*}\left(1,\left(\frac{u}{1-u}\right)^{\frac{1}{t^*}}\right)  & = & \frac{1}{2},
\end{eqnarray*}
because $M_1$ is the arithmetic mean. When $u > 1/2$, $u/(1-u) > 1$ and so this time
\begin{eqnarray*}
  \lim_{t^* \rightarrow 0^+} M_{1-t^*}\left(1,\left(\frac{u}{1-u}\right)^{\frac{1}{t^*}}\right)  & = & + \infty.
\end{eqnarray*}
Hence,
\begin{eqnarray}
\partialinftloss{1}(u) & = & 2 \cdot \iver{u \leq 1/2},
  \end{eqnarray}
which is (twice) the partial loss of the 0/1 loss \cite{rwID}.
\end{proof}
This ends the proof of Theorem \ref{th-CPE}.

\newpage

\section{Supplementary material on experiments} \label{sec-sup-exp}

\subsection{Domains}\label{sec-doms}

Table \ref{t-s-domains} presents the 10 domains we used for our experiments.

\begin{table}[t]
\begin{center}
\begin{tabular}{|cc|r|r|l|}
  \hline \hline
  Domain & Source & \multicolumn{1}{c|}{$m$} & \multicolumn{1}{c|}{$d$} & \\ \hline
  \domainname{sonar} & UCI & 208 & 60 & {\tiny \url{https://archive.ics.uci.edu/ml/datasets/wine+quality}}\\
  \domainname{winered} & UCI & 1 599 & 12 & {\tiny \url{https://archive.ics.uci.edu/ml/datasets/wine+quality}}\\
  \domainname{abalone} & UCI & 4 177 & 9 & {\tiny \url{https://archive.ics.uci.edu/ml/datasets/abalone}}\\
  \domainname{qsar} & UCI & 1 055 & 41 & {\tiny \url{https://archive.ics.uci.edu/ml/datasets/QSAR+biodegradation}}\\
  \domainname{winewhite} & UCI & 4 898 & 12 & {\tiny \url{https://archive.ics.uci.edu/ml/datasets/wine+quality}}\\
  \domainname{hillnonoise} & UCI & 1 212 & 101 & {\tiny \url{http://archive.ics.uci.edu/ml/datasets/hill-valley}}\\
  \domainname{hillnoise} & UCI & 1 212 & 101 & {\tiny \url{http://archive.ics.uci.edu/ml/datasets/hill-valley}}\\
  \domainname{eeg} & UCI & 14 980 & 15 & {\tiny \url{https://archive.ics.uci.edu/ml/datasets/EEG+Eye+State}}\\
  \domainname{creditcard}$^{*}$ & UCI & 14 599 & 24 & {\tiny \url{https://archive.ics.uci.edu/ml/datasets/default+of+credit+card+clients}}\\
  \domainname{adult} & UCI & 32 561 & 15 & {\tiny \url{https://archive.ics.uci.edu/ml/datasets/adult}}\\
\hline\hline
\end{tabular} 
\end{center}
  %\vspace{-0.3cm}
\caption{Public domains considered in our experiments ($m=$ total number
  of examples, $d=$ total number of example's features, including class), ordered in
  increasing $m \times d$ (see text). (*) first $m$ rows in the domain.}
  \label{t-s-domains}
\end{table}

\subsection{Implementation details and full set of experiments on linear combinations of decision trees}\label{sec-full-exp}

\paragraph{Summary} This Section depicts the full set of experiments summarized in Table \ref{tab:stattests} (\mainfile), from Table \ref{tab:plots-sonar} to Table \ref{tab:plots-adult}. Tables are ordered in increasing size of the domain (Table \ref{t-s-domains}). In all cases, up to $J=20$ trees have been trained, of size 15 (total number of nodes, except the two biggest domains, for which the size is 5). For all datasets, except \domainname{creditcard}
and \domainname{adult}, we have tested $t$ in the complete range, $t\in \{0.0, 0.2, 0.4, 0.6, 0.8, 0.9, 1.0, 1.1\}$ (the \mainfile~only reports results for $t\geq 0.6$), and in all cases, models both clamped and not clamped. For each dataset, we have set a 10-folds stratified cross-validation experiment, and report the averages for readability (Table \ref{tab:stattests} in \mainfile~gives the results of a Student paired $t$-test on error averages for comparison, limit $p$-val = 0.1). We also provide two examples of training error averages for domains \domainname{hillnoise} and \domainname{hillnonoise} (Tables \ref{tab:plots-terr-hillnonoise} and \ref{tab:plots-terr-hillnoise}).

\paragraph{Implementation details of \tadaboost} First, regarding file format, we only input a \texttt{.csv} file to \tadaboost. We do not specify a file with feature types as in ARFF files. \tadaboost~recognizes the type of each feature from its column content and distinguishes two main types of features: numerical and categorical. The distinction is important to design the splits during decision tree induction: for numerical values, splits are midpoints between two successive observed values. For categorical, splits are partitions of the feature values in two non-empty subsets. Our implementation of \tadaboost~(programmed in Java) makes it possible to choose $t$ not just in the range of values for which we have shown that boosting-compliant convergence is possible ($t \in [0,1]$), but also $t>1$. Because we thus implement \adaboost~($t=1$) but also for $t>1$, weights can fairly easily become infinite, we have implemented a safe-check during training, counting the number of times the weights become infinite or zero (note that in this latter case, this really is a problem just for \adaboost~because in theory this should never happen unless the weak classifiers achieve perfect (or perfectly wrong) classification), but also making sure leveraging coefficients for classifiers do not become infinite for \adaboost, a situation that can happen because of numerical approximations in encoding. In our experiments, we have observed that none of these problematic cases did occur (notice that this could not be the case if we were to boost for a large number of iterations). We have implemented algorithm \tadaboost~exactly as specified in \mainfile. The weak learner is implemented to train a decision tree in which the stopping criterion is the size of the tree reaching a user-fixed number of nodes. There is thus no pruning. Also, the top-down induction algorithm proceeds by iteratively picking the heaviest leaf in the tree and then choosing the split that minimizes the expected Bayes risk of the tempered loss, computing using the same $t$ values as for \tadaboost, and with the constraint to not get pure leaves (otherwise, the real prediction at the leaves, which relies on the link of the loss, would be infinite for \adaboost). In our implementation of decision-tree induction, when the number of possible splits exceeds a fixed number $S$ (currently, 2 000), we pick the best split in a subset of $S$ splits picked at random.

\paragraph{Results} First, one may notice in several plots that the average test error increases with the number of trees. This turns out to be a sign of overfitting, as exemplified for domains \domainname{hillnonoise} and \domainname{hillnoise}, for which we provide the training curves. If we align the training curves at $T=1$ (the value is different because the splitting criterion for training the tree is different), we notice that the experimental convergence on training is similar for all values of $t$ (Tables \ref{tab:plots-terr-hillnonoise} and \ref{tab:plots-terr-hillnoise}). The other key experimental result, already visible from Table \ref{tab:stattests} (\mainfile), is that pretty much all tested values of $t$ are necessary to get the best results. One could be tempted to conclude that $t$ slightly smaller than $1.0$ seems to the a good fit from Table \ref{tab:stattests} (\mainfile), but the curves show that this is more a consequence of the Table being computed for $J=20$ trees. The case of \domainname{eeg} illustrates best this phenomenon: while small $t$-values are clearly the best when there is no noise, the picture is completely reversed when there is training noise. Notice that this ordering is almost reversed on \domainname{creditcard} and \domainname{adult}: when there is noise, small values of $t$ tend to give better results. Hence, in addition to getting (i) a pruning mechanism that works for all instances of the tempered loss and (ii) a way to guess the right number of models in the ensemble, a good problem to investigate is in fact appropriately tuning $t$ in a domain-dependent way. Looking at all plots reveals that substantial gains could be obtained with an accurate procedure (over the strategy that would be to always pick a fixed $t$, \textit{e.g.} $t=1$).

  \setlength\tabcolsep{0pt}
  \newcommand{\dname}{sonar}
\newcommand{\nonoisetag}{May_13th__9h_11m_16s}
\newcommand{\noisetag}{May_13th__9h_44m_44s}
\begin{table*}
  \centering
 \includegraphics[trim=0bp 0bp 0bp 0bp,clip,width=\columnwidth]{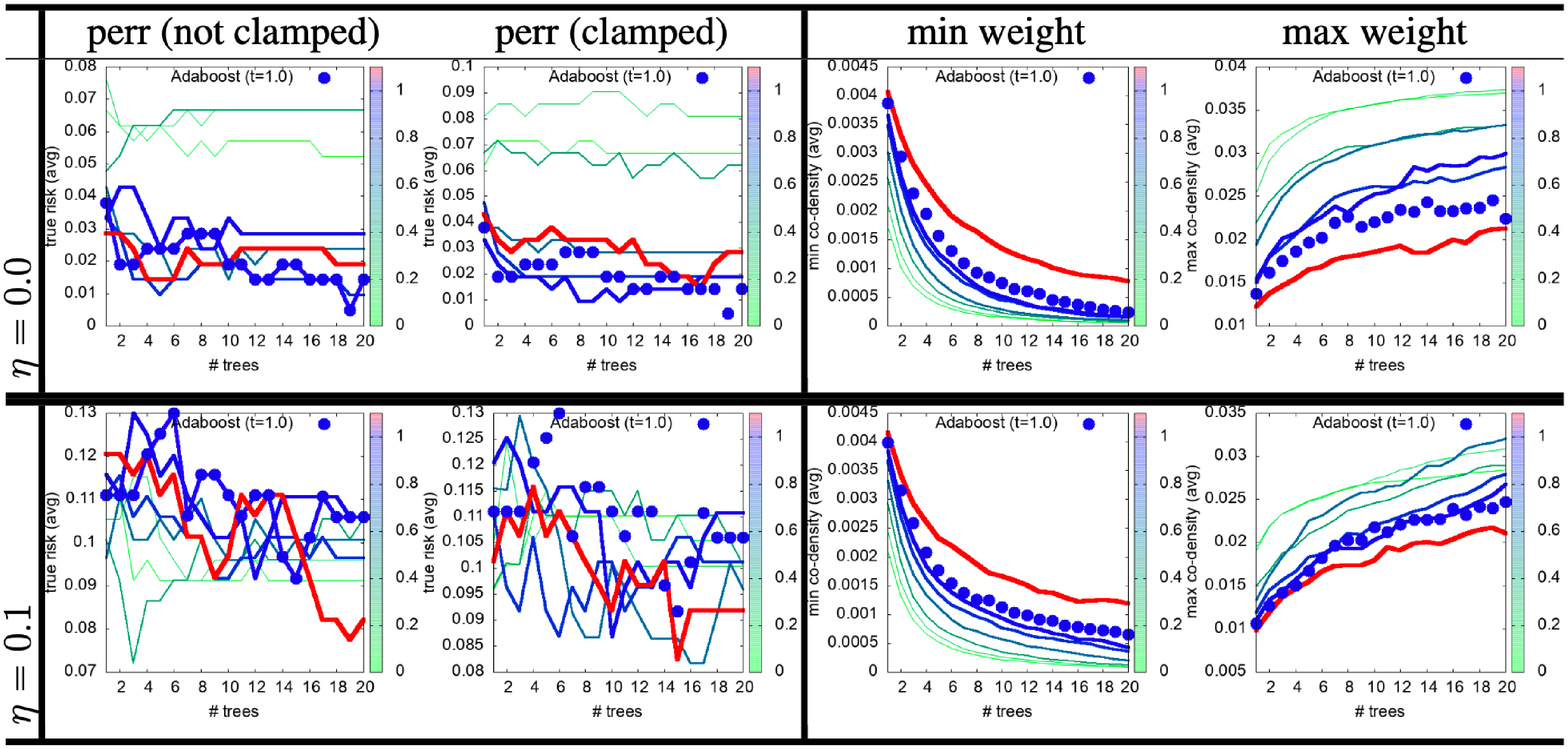}
\caption{Experiments on \tadaboost~comparing with \adaboost~($t=1$, bullets) on domain \domainname{\dname}, when trained without noise ($\eta = 0.0$, top row) and with noise ($\eta = 0.1$, bottom row). Columns are, from left to right, the estimated true error of non-clamped and clamped models, and the min and max codensity weights. The set of $t$ values used is displayed in each plot with a colormap (right), and varying thickness of curves for an additional ease of reading (the thicker the curve, the larger $t$). \adaboost's reference results are displayed with bullets. Averages shown for readability.}
    \label{tab:plots-\dname}
  \end{table*}

\renewcommand{\dname}{winered}
\renewcommand{\nonoisetag}{May_12th__16h_58m_28s}
\renewcommand{\noisetag}{May_12th__18h_14m_1s}
\begin{table*}
  \centering
   \includegraphics[trim=0bp 0bp 0bp 0bp,clip,width=\columnwidth]{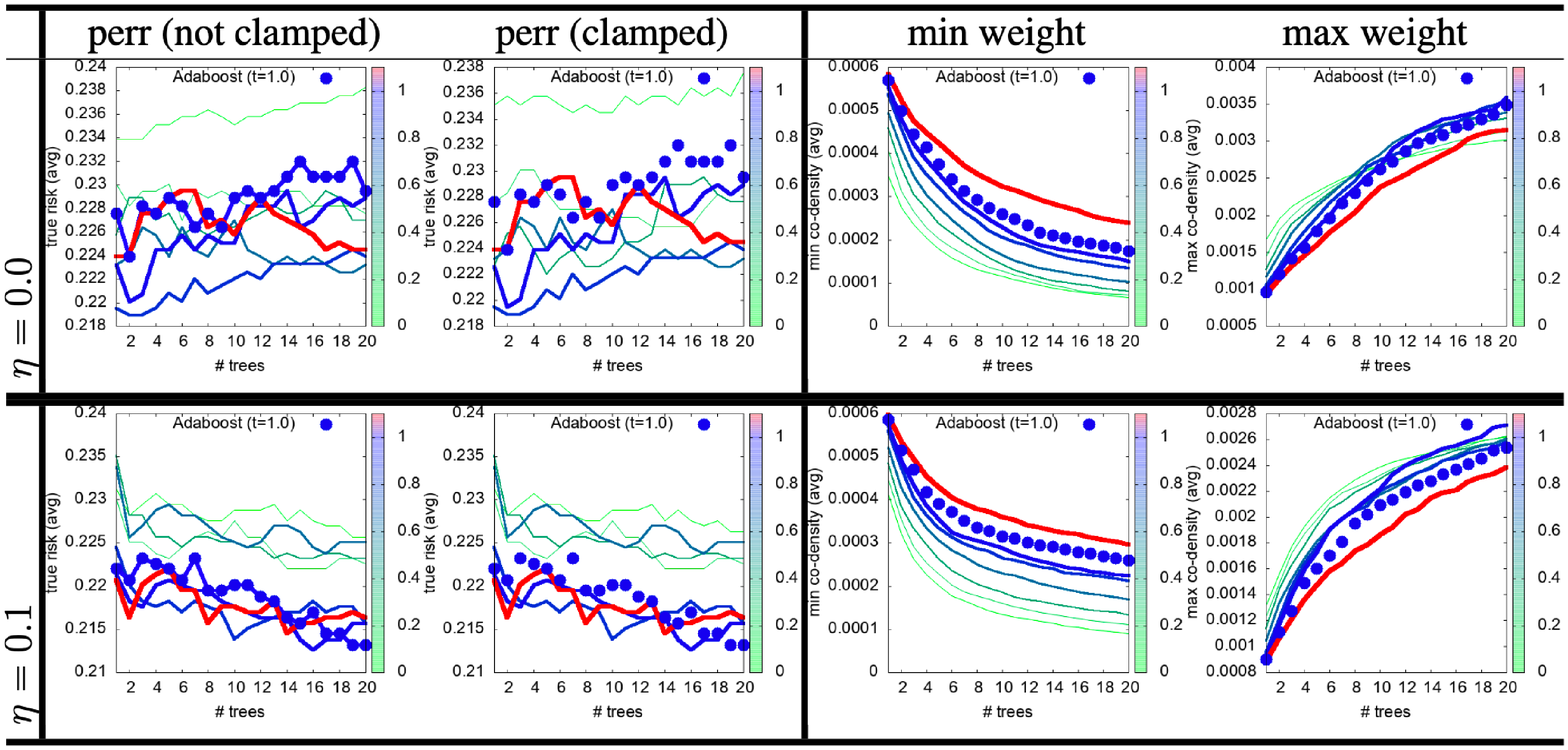}
\caption{Experiments on \tadaboost~comparing with \adaboost~($t=1$, bullets) on domain \domainname{\dname}. Conventions follow Table \ref{tab:plots-sonar}.}
    \label{tab:plots-winered}
  \end{table*}

 \renewcommand{\dname}{abalone}
\renewcommand{\nonoisetag}{May_12th__20h_23m_7s}
\renewcommand{\noisetag}{May_13th__6h_2m_28s}
\begin{table*}
  \centering
   \includegraphics[trim=0bp 0bp 0bp 0bp,clip,width=\columnwidth]{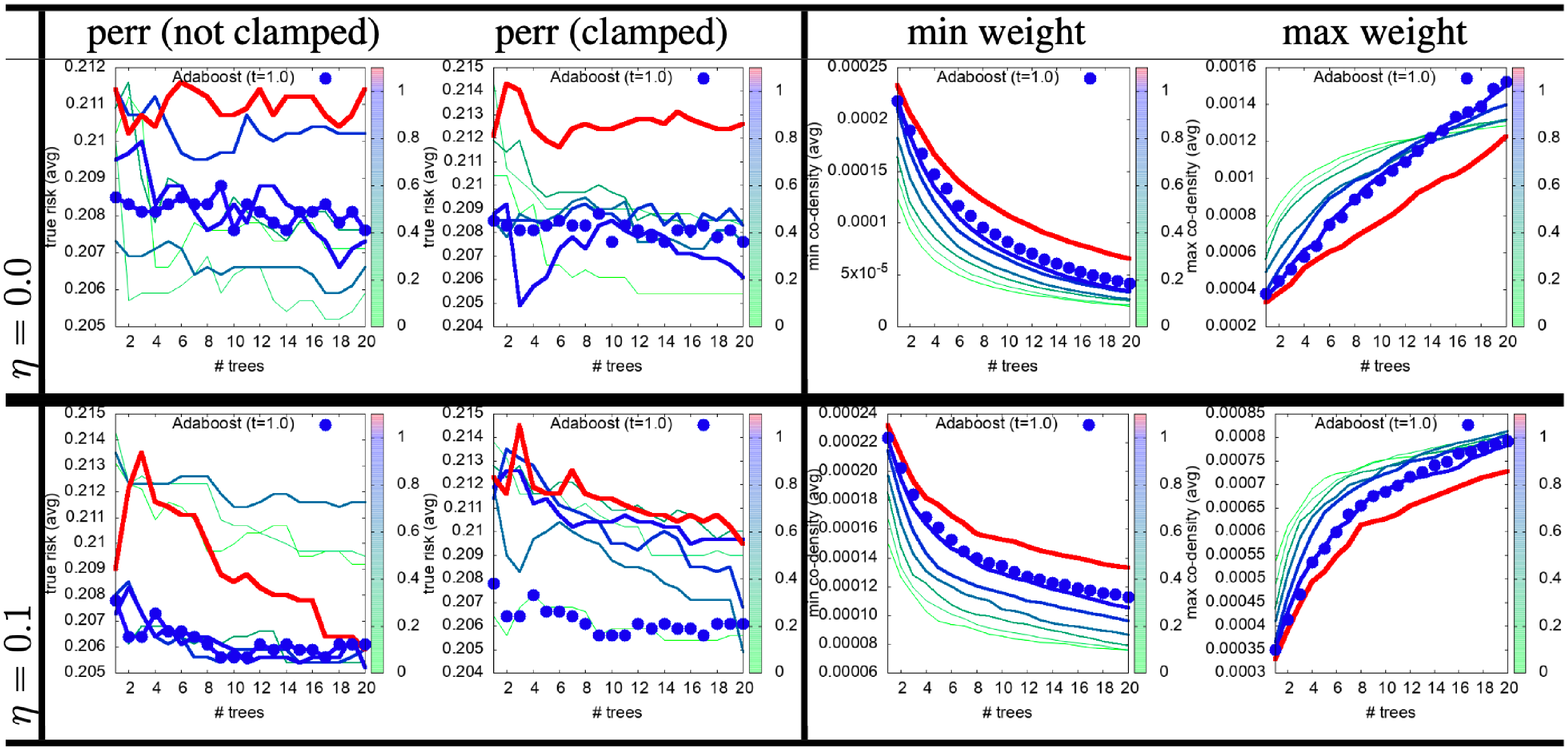}
\caption{Experiments on \tadaboost~comparing with \adaboost~($t=1$, bullets) on domain \domainname{\dname}. Conventions follow Table \ref{tab:plots-sonar}.}
    \label{tab:plots-\dname}
  \end{table*}

\renewcommand{\dname}{qsar}
\renewcommand{\nonoisetag}{May_12th__19h_29m_11s}
\renewcommand{\noisetag}{May_16th__4h_59m_22s}
\begin{table*}
  \centering
   \includegraphics[trim=0bp 0bp 0bp 0bp,clip,width=\columnwidth]{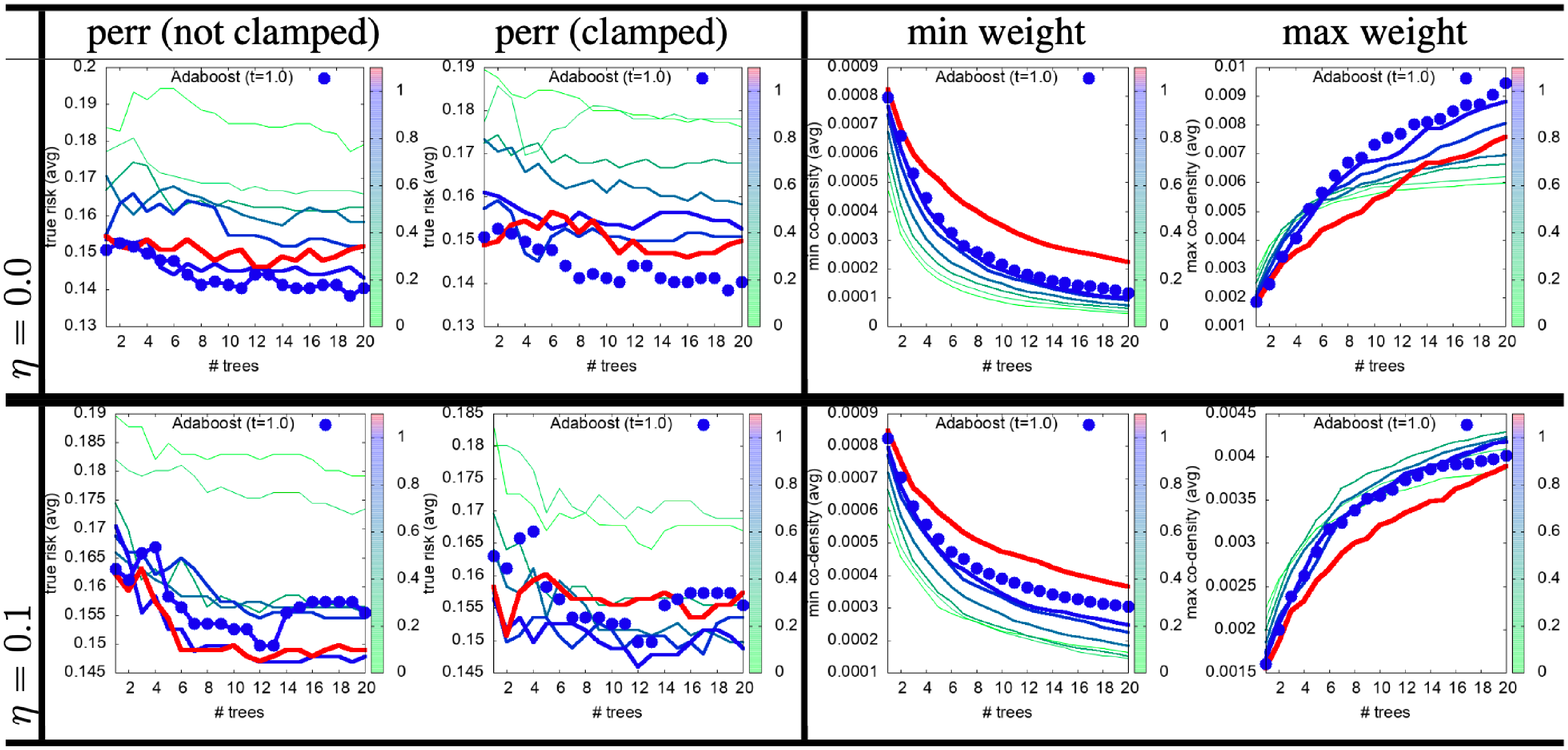}
\caption{Experiments on \tadaboost~comparing with \adaboost~($t=1$, bullets) on domain \domainname{\dname}. Conventions follow Table \ref{tab:plots-sonar}.}
    \label{tab:plots-\dname}
  \end{table*}
  
 \renewcommand{\dname}{winewhite}
\renewcommand{\nonoisetag}{May_13th__10h_44m_37s}
\renewcommand{\noisetag}{May_13th__16h_34m_33s}
\begin{table*}
  \centering
   \includegraphics[trim=0bp 0bp 0bp 0bp,clip,width=\columnwidth]{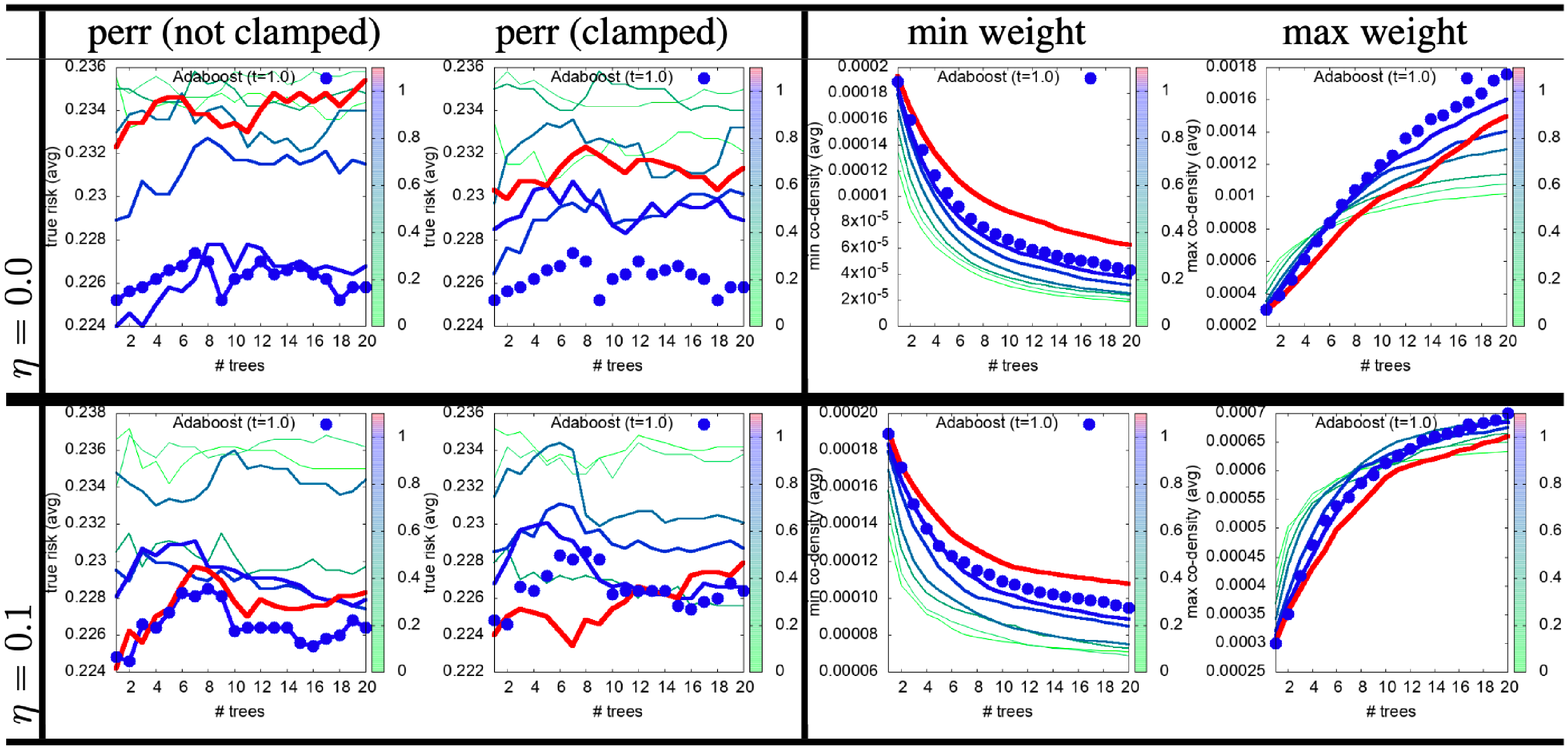}
\caption{Experiments on \tadaboost~comparing with \adaboost~($t=1$, bullets) on domain \domainname{\dname}. Conventions follow Table \ref{tab:plots-sonar}.}
    \label{tab:plots-\dname}
  \end{table*}

\renewcommand{\dname}{hillnonoise}
\renewcommand{\nonoisetag}{May_15th__17h_37m_2s}
\renewcommand{\noisetag}{May_15th__20h_35m_11s}
\begin{table*}
  \centering
    \includegraphics[trim=0bp 0bp 0bp 0bp,clip,width=\columnwidth]{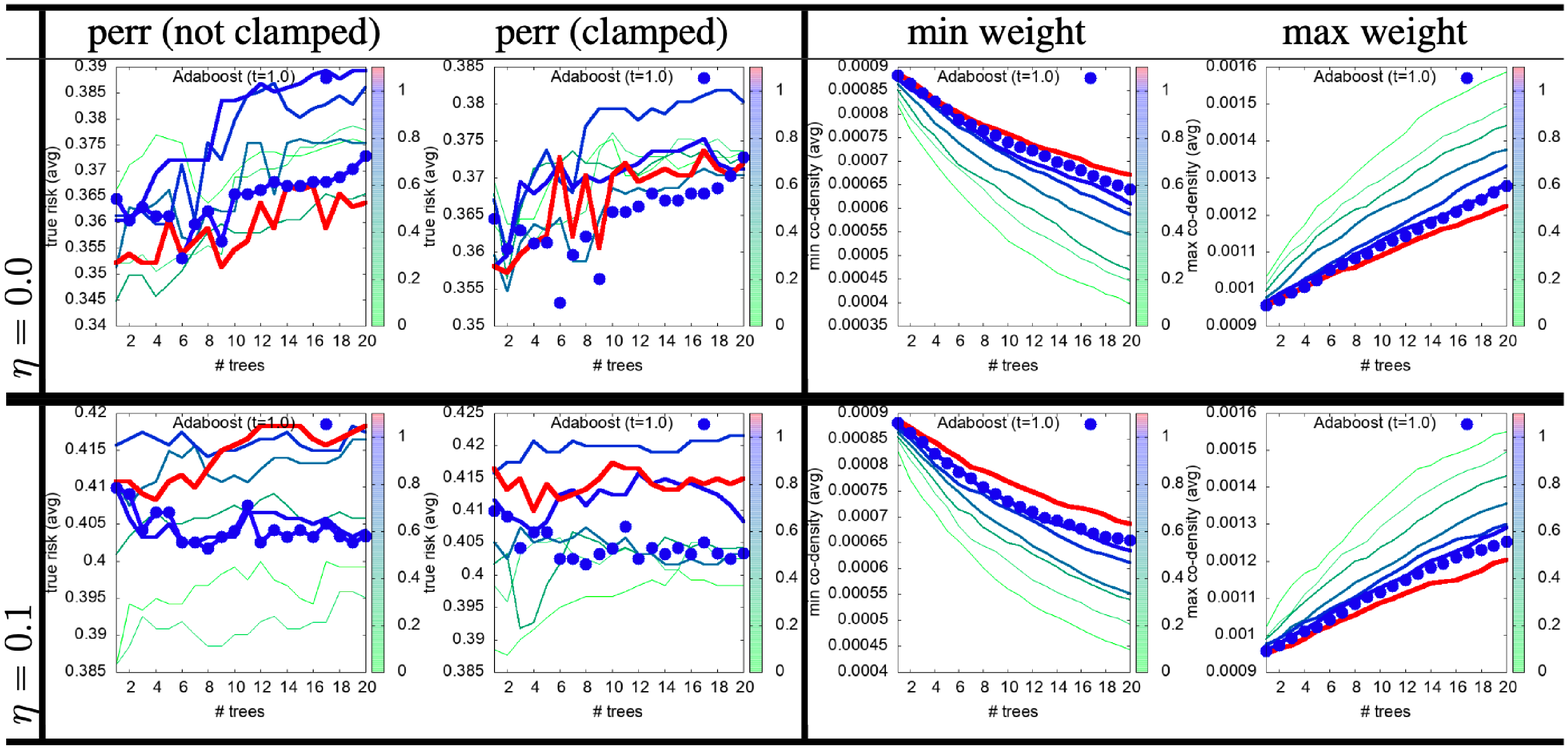}
\caption{Experiments on \tadaboost~comparing with \adaboost~($t=1$, bullets) on domain \domainname{\dname}. Conventions follow Table \ref{tab:plots-sonar}.}
    \label{tab:plots-\dname}
  \end{table*}

  \begin{table*}
  \centering
    \includegraphics[trim=0bp 0bp 0bp 0bp,clip,width=\columnwidth]{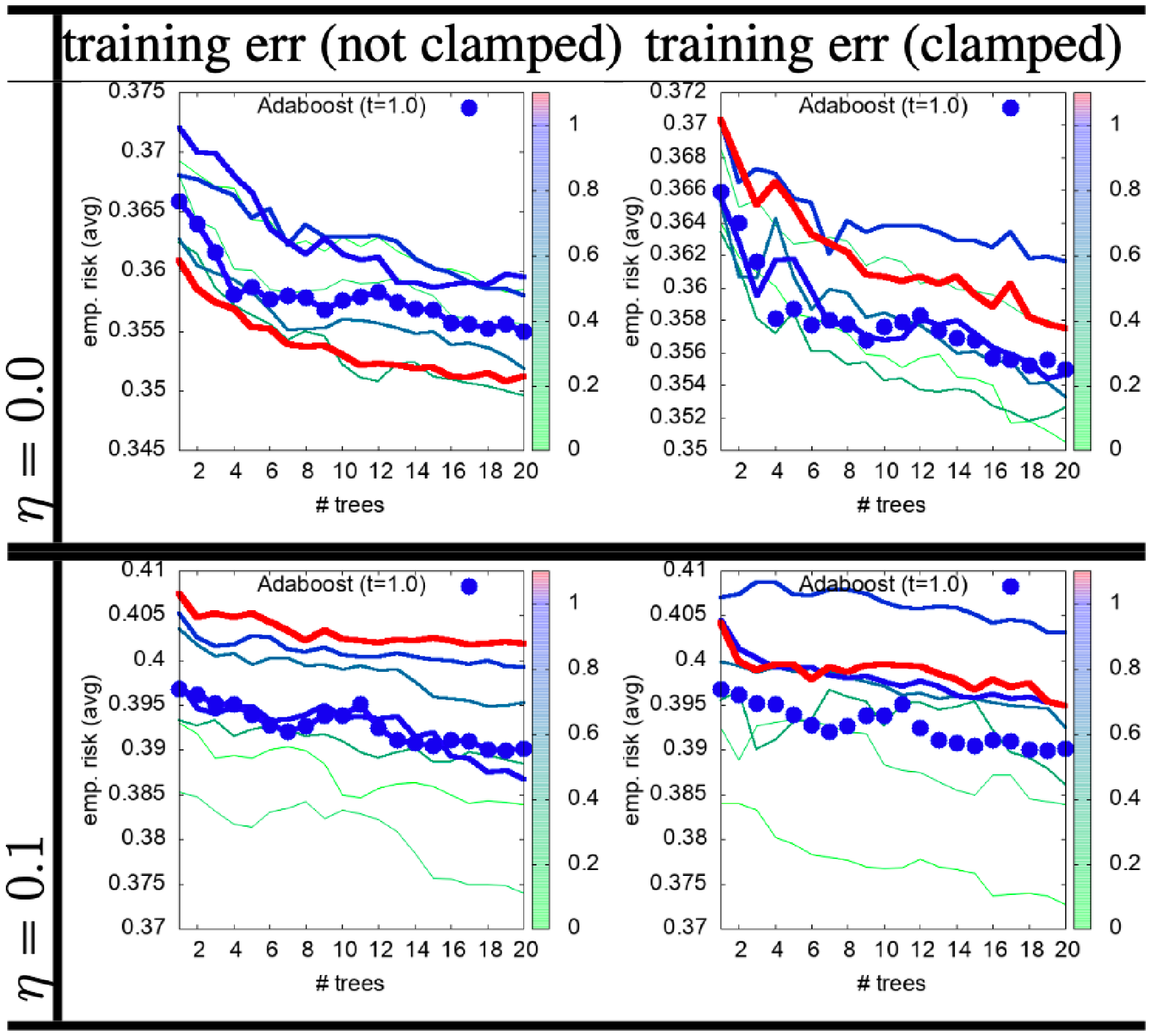}
\caption{Experiments on \tadaboost~comparing with \adaboost~($t=1$, bullets) on domain \domainname{\dname}: training errors displayed for all algorithms using conventions from Table \ref{tab:plots-sonar}. See text for details.}
    \label{tab:plots-terr-\dname}
  \end{table*}

  \renewcommand{\dname}{hillnoise}
\renewcommand{\nonoisetag}{May_16th__4h_0m_13s}
\renewcommand{\noisetag}{May_16th__7h_41m_38s}
\begin{table*}
  \centering
  \includegraphics[trim=0bp 0bp 0bp 0bp,clip,width=\columnwidth]{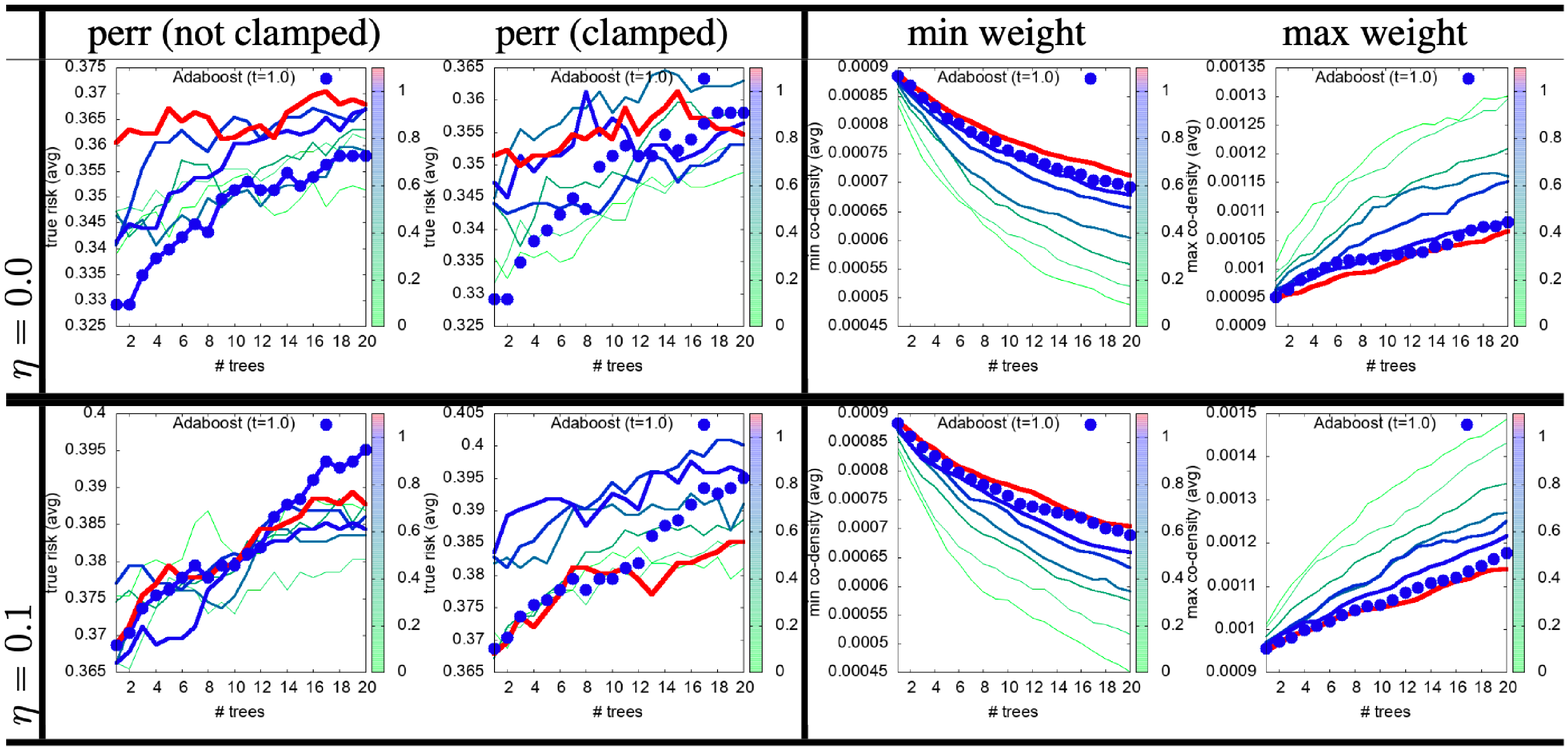}
\caption{Experiments on \tadaboost~comparing with \adaboost~($t=1$, bullets) on domain \domainname{\dname}. Conventions follow Table \ref{tab:plots-sonar}.}
    \label{tab:plots-\dname}
  \end{table*}

  \begin{table*}
  \centering
  \includegraphics[trim=0bp 0bp 0bp 0bp,clip,width=\columnwidth]{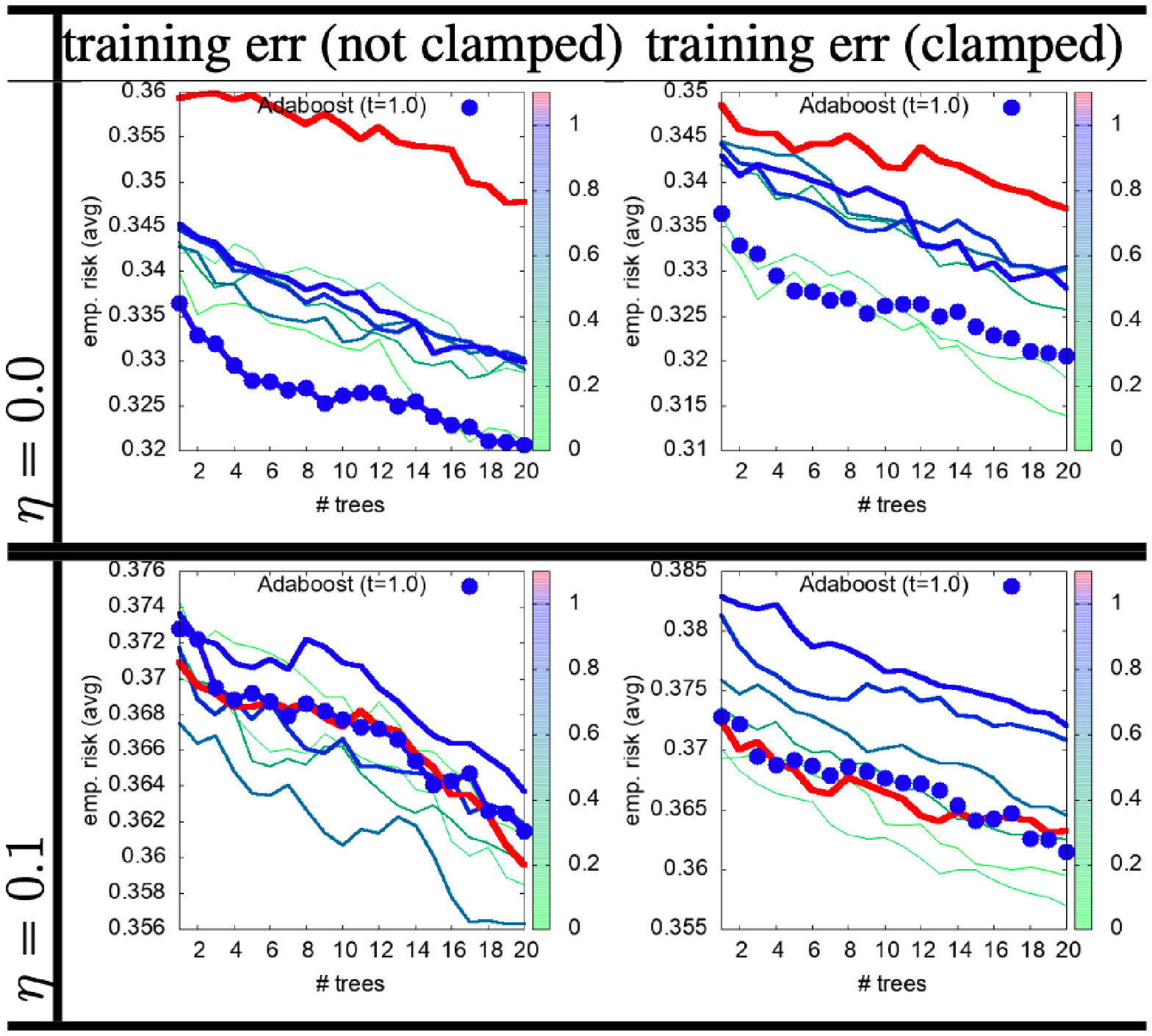}
\caption{Experiments on \tadaboost~comparing with \adaboost~($t=1$, bullets) on domain \domainname{\dname}: training errors displayed for all algorithms using conventions from Table \ref{tab:plots-sonar}. See text for details.}
    \label{tab:plots-terr-\dname}
  \end{table*}

  \renewcommand{\dname}{eeg}
\renewcommand{\nonoisetag}{May_13th__10h_44m_11s}
\renewcommand{\noisetag}{May_14th__12h_16m_28s}
\begin{table*}
  \centering
  \includegraphics[trim=0bp 0bp 0bp 0bp,clip,width=\columnwidth]{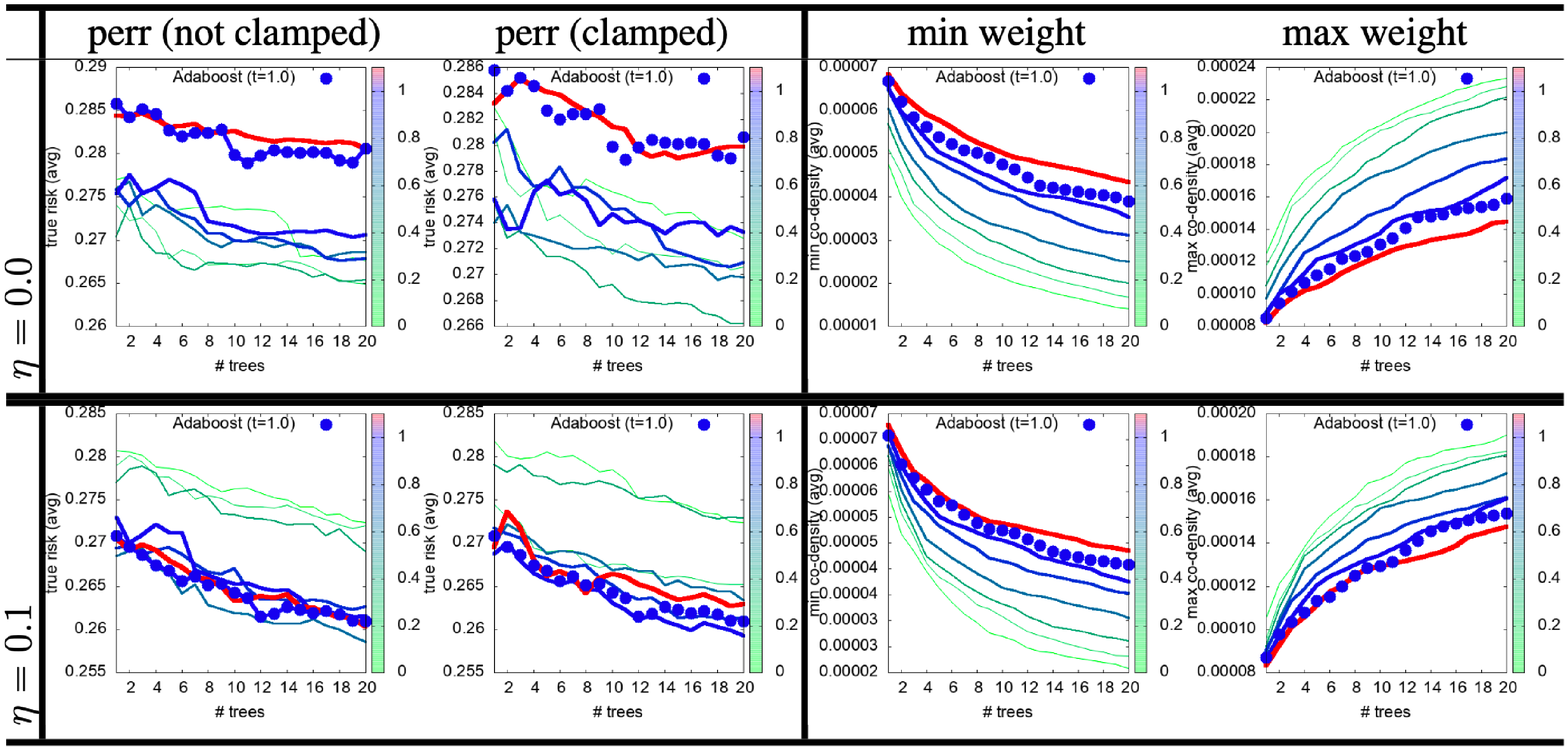}
\caption{Experiments on \tadaboost~comparing with \adaboost~($t=1$, bullets) on domain \domainname{\dname}. Conventions follow Table \ref{tab:plots-sonar}.}
    \label{tab:plots-\dname}
  \end{table*}

  \renewcommand{\dname}{creditcard}
\renewcommand{\nonoisetag}{May_15th__10h_21m_37s}
\renewcommand{\noisetag}{May_16th__3h_56m_51s}
\begin{table*}
  \centering
  \includegraphics[trim=0bp 0bp 0bp 0bp,clip,width=\columnwidth]{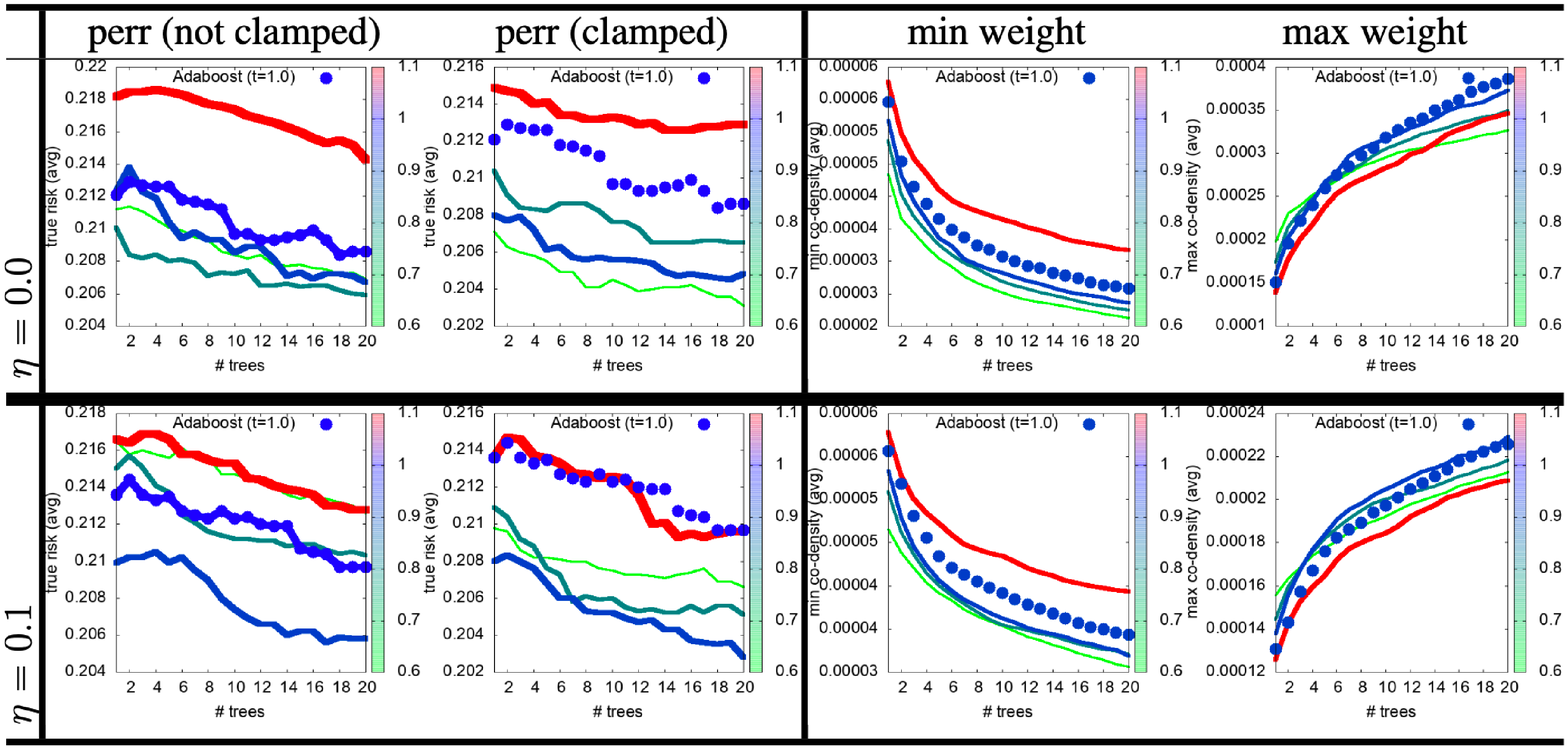}
\caption{Experiments on \tadaboost~comparing with \adaboost~($t=1$, bullets) on domain \domainname{\dname}. Conventions follow Table \ref{tab:plots-sonar}.}
    \label{tab:plots-\dname}
  \end{table*}

   \renewcommand{\dname}{adult}
\renewcommand{\nonoisetag}{May_12th__14h_55m_45s}
\renewcommand{\noisetag}{May_14th__5h_52m_43s}
\begin{table*}
  \centering
  \includegraphics[trim=0bp 0bp 0bp 0bp,clip,width=\columnwidth]{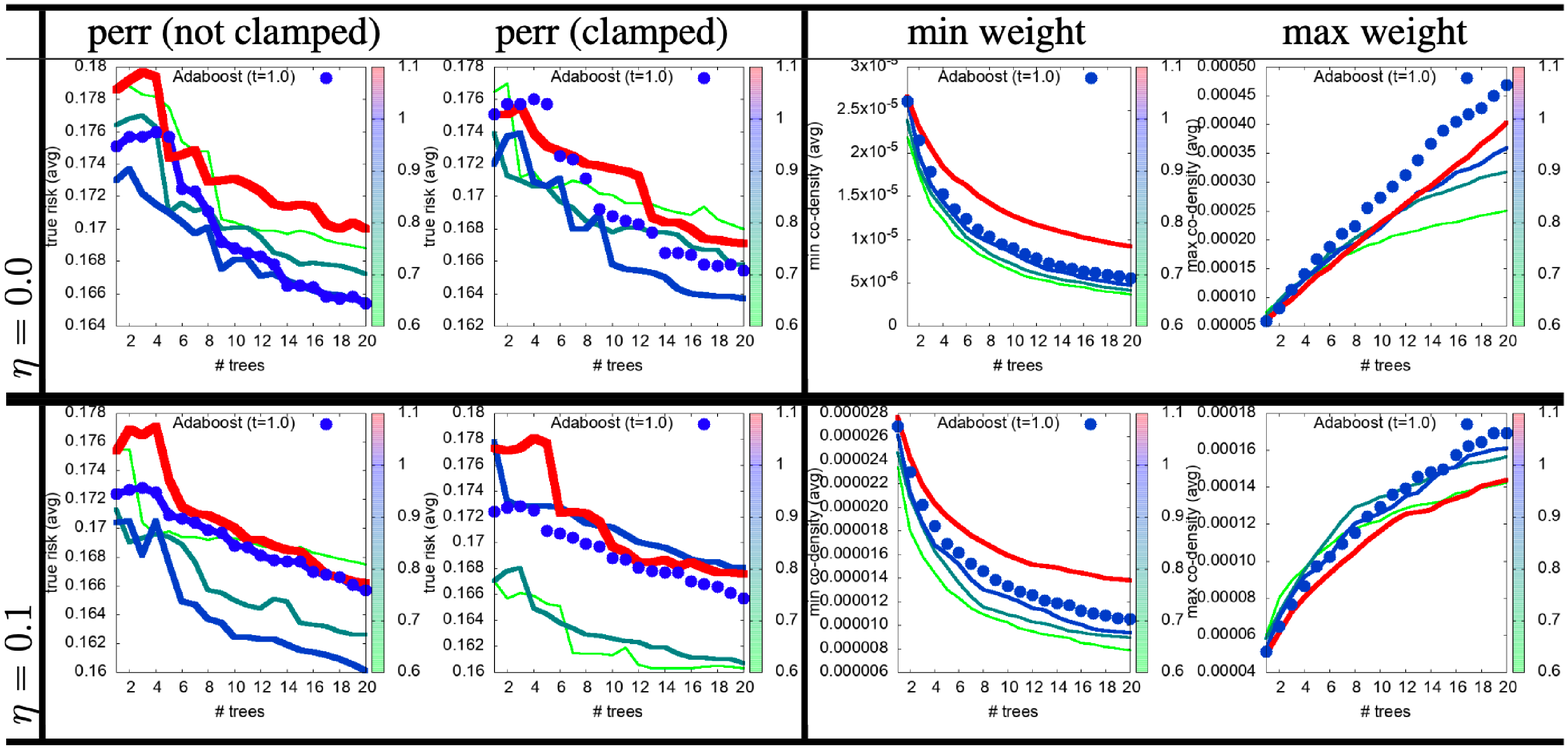}
\caption{Experiments on \tadaboost~comparing with \adaboost~($t=1$, bullets) on domain \domainname{\dname}. Conventions follow Table \ref{tab:plots-sonar}.}
    \label{tab:plots-\dname}
  \end{table*}

\end{document}